\documentclass{article}

\usepackage[preprint]{neurips_2025}

\usepackage{amsmath,amssymb,amsthm}
\usepackage{mathtools}
\usepackage{booktabs}
\usepackage{microtype}
\usepackage{hyperref}
\usepackage{url}
\usepackage{graphicx}
\usepackage{xcolor}
\usepackage{algorithm}
\usepackage{algorithmic}
\usepackage{bbm}
\usepackage{multirow}
\usepackage{pifont}

\newcommand{\R}{\mathbb{R}}
\newcommand{\E}{\mathbb{E}}
\newcommand{\KL}{\mathrm{KL}}
\newcommand{\sg}{\mathrm{sg}}          % stop-gradient
\newcommand{\supp}{\mathrm{supp}}
\newcommand{\Var}{\mathrm{Var}}
\newcommand{\cF}{\mathcal{F}}
\newcommand{\cL}{\mathcal{L}}
\newcommand{\cP}{\mathcal{P}}

\newcommand{\norm}[1]{\left\|#1\right\|}
\DeclareMathOperator*{\argmin}{arg\,min}

\theoremstyle{plain}
\newtheorem{theorem}{Theorem}
\newtheorem{proposition}[theorem]{Proposition}
\newtheorem{corollary}[theorem]{Corollary}
\newtheorem{lemma}[theorem]{Lemma}
\theoremstyle{definition}
\newtheorem{definition}[theorem]{Definition}
\theoremstyle{remark}
\newtheorem{remark}[theorem]{Remark}
\title{A Unified Framework for Data-Free One-Step Sampling via Wasserstein Gradient Flows}

\author{%
  Chenguang Wang \\
  School of Data Science \\
  The Chinese University of Hong Kong, Shenzhen \\
  \texttt{chenguangwang@link.cuhk.edu.cn} \\
  \And
  Tianshu Yu \\
  School of Data Science \\
  The Chinese University of Hong Kong, Shenzhen \\
  \texttt{yutianshu@cuhk.edu.cn}
}

\begin{document}
\maketitle

% ============================================================
\begin{abstract}
% ============================================================
We develop a unified theoretical framework for data-free one-step sampling from unnormalized
target distributions based on Wasserstein gradient flows. For a broad class of standard
f-divergence objectives, we show that the induced velocity field admits the universal form
$\mathbf{V}(x)=w(r(x))\,\beta(x)$, where $\beta(x)=\nabla \log (p(x)/q(x))$ is shared across
objectives and $w$ is determined solely by the choice of divergence. This decomposition shows
that standard f-divergence drifts share the same asymptotic target distribution $p$ and differ
primarily in how they redistribute transient repair effort across under-covered regions. To
formalize this distinction, we derive a one-step regional-response theory for a soft
under-coverage functional and obtain a compression--elasticity identity that links divergence
choice to the geometry of mass transport into under-covered regions. We further extend the
framework beyond the f-divergence family to the Log-Variance (LV) divergence, analyze how the
reference distribution alters the resulting drift structure, and motivate a practical LV-inspired
surrogate for data-free training. Based on this theory, we instantiate the framework with a
KDE-based implementation and describe a complementary normalizing-flow route, enabling one-step
inference after training. Experiments on multimodal Gaussian-mixture benchmarks are consistent
with the theoretical predictions and demonstrate effective one-step sampling on these targets.
\end{abstract}

% ============================================================
\section{Introduction}
\label{sec:intro}
% ============================================================

Sampling from unnormalized target distributions $p(x) \propto \exp(-E(x))$ is a foundational
problem in Bayesian inference, molecular simulation, and statistical physics. Classical
approaches---Markov Chain Monte Carlo~\cite{andrieu2003introduction,neal2011mcmc}, Langevin
dynamics~\cite{welling2011bayesian}, and Hamiltonian Monte Carlo~\cite{duane1987hybrid}---provide
asymptotically exact samples but rely on sequential iterative inference: each sample requires
hundreds to thousands of gradient evaluations, making these methods prohibitively expensive on
complex, multimodal energy landscapes.

Neural samplers aim to amortize this cost by learning a generator whose forward pass produces
approximate samples at test time. Boltzmann generators~\cite{noe2019boltzmann} and stochastic
normalizing flows~\cite{wu2020stochastic} learn transport maps to the target by minimizing
$\KL(q \| p)$, enabling single-pass generation when the target can be represented well by the
chosen invertible family. Their effectiveness on multimodal targets therefore depends strongly on
model expressivity and optimization, while bijective architectural constraints can limit
scalability in high dimensions. Diffusion-based samplers---including denoising diffusion
samplers~\cite{vargasdenoising}, controlled Monte Carlo diffusions~\cite{vargastransport},
iDEM~\cite{akhound2024iterated}, and improved off-policy variants~\cite{sendera2024improved}---learn
stochastic transport trajectories from the prior to the target and achieve strong empirical
performance on multimodal benchmarks, but they retain multi-step inference and typically require
hundreds of sequential network evaluations per sample.

Reducing inference to a single function evaluation (1-NFE) is therefore a central objective.
One route is distillation: \citet{zhang2024efficient} adapt consistency
models~\cite{song2023consistency} to the Boltzmann sampling setting, compressing a pre-trained
multi-step diffusion sampler into a one-step model, while \citet{jutras2025one}
remove the external teacher via self-distillation but still simulate a reverse diffusion process
at training time. These methods reduce inference cost, but they continue to inherit the
multi-step framework during training and provide no principled account of how a one-step update
repairs under-covered regions of a multimodal target. The Drifting Model~\citep{deng2026drifting}
offers a complementary route to 1-NFE sampling: instead of distilling a multi-step sampler, it
trains a generator by exploiting the training-time evolution of its pushforward distributions,
guided by a \emph{drifting field} that continuously transports $q$ toward $p$. This removes
inference-time iteration entirely, but the original formulation is restricted to the data-based
setting, where positive samples from the target are available. Extending this paradigm to the
data-free setting, where only the energy $E(x)$ or its gradient is accessible, remains open.
More fundamentally, the relationship between the drifting field and the underlying divergence
objective has not been systematically characterized, especially with respect to how different
objectives redistribute one-step repair effort across a multimodal landscape.

We address these limitations by developing a unified data-free framework grounded in Wasserstein
gradient flows. By computing the Wasserstein gradient flow velocity field for a broad class of
standard f-divergence functionals $\cF(q) = D_f(p \| q)$ via the JKO theorem~\cite{jordan1998variational},
we derive a universal formula $\mathbf{V}(x) = w(r)\,\beta(x)$, where
$\beta(x)=\nabla\log(p(x)/q(x))$ is shared across objectives and $w(r)$ is determined solely by
the choice of divergence. This formula recovers the Reverse-KL drifting direction underlying
kernelized drifting methods, extends it to the broader f-divergence family, and clarifies
that standard f-divergence drifts differ not in their asymptotic target, which remains $p$, but
in how they spatially reweight transient repair. We formalize this distinction through a local
one-step theory of regional repair that links divergence choice to the geometry of mass transport
into under-covered regions. We further extend the framework beyond the f-divergence family to the
Log-Variance (LV) divergence~\cite{richter2020vargrad,nusken2021solving,richter2022robust,richter2024improved},
which is widely used in neural sampler training~\cite{richter2024improved}. In that case, the
choice of reference distribution changes the drift structure, and the exact variational analysis
motivates a practical LV-inspired surrogate for data-free implementation.

We evaluate the resulting methods on multimodal Gaussian-mixture benchmarks spanning
two-dimensional mode coverage and higher-dimensional mode-bridging. These experiments provide a
data-free testbed for the unified framework and for the practical implementations developed in
later sections.

The contributions of this work are as follows.
\begin{enumerate}
  \item \textbf{Unified drifting field formula.} For a broad class of standard f-divergence functionals $\cF(q) = D_f(p \| q)$,
    we derive the Wasserstein gradient flow velocity field as $\mathbf{V}(x) = w(r) \cdot \beta(x)$,
    where $w(r) = r^2 f''(r)$, $r(x) = p(x)/q(x)$, and $\beta(x) = \nabla \log r(x)$.
    All f-divergences share the same directional signal $\beta$ and differ only in how they
    weight it via $w(r)$. We further extend this framework to the Log-Variance (LV) divergence,
    which lies outside the f-divergence family, and show how the choice of reference
    distribution changes the resulting drift structure.
  \item \textbf{Theoretical analysis of regional mode repair.} We derive a one-step
    regional-response formula for a soft under-coverage functional and obtain a
    compression--elasticity identity
    \[
      G_f(t;\delta,\varepsilon)
      =
      \int_{\Omega_t} p\,q_t\,w(r)\,\|\beta\|^2\,[\kappa_t-\eta_f(r)]\,dx,
      \qquad
      \eta_f(r)=\frac{r w'(r)}{w(r)}.
    \]
    This characterizes divergence-dependent transient repair bias, yields explicit repair
    thresholds for Reverse KL, Forward KL, $\chi^2$, and Tsallis drifts, and clarifies that the
    classical reverse-KL/forward-KL dichotomy should be interpreted here as a difference in
    transient regional response rather than in asymptotic target.
  \item \textbf{Data-free implementation routes.} We instantiate the framework with a KDE-based
    implementation, which estimates $\log q$ and $\nabla\log q$ directly from the current
    particles, and we describe a complementary normalizing-flow route, which can provide the same
    quantities analytically after fitting an invertible or auxiliary density model. For LV
    objectives, we further introduce a centered-log-ratio-gated surrogate compatible with the
    same sampler--estimator interface.
\end{enumerate}

Together, these results turn the drifting-field perspective into a unified data-free framework
for designing, analyzing, and implementing one-step samplers from unnormalized targets.

% ============================================================
\section{Background}
\label{sec:background}
% ============================================================

\subsection{Sampling as Wasserstein Gradient Flow}
\label{sec:wgf}

We work in the space $\cP_2(\R^d)$ of probability measures on $\R^d$ with finite second moments.
The central perspective is to formulate sampling as optimization in distribution space.
Given a target $p$ and a divergence functional $\cF: \cP_2 \to \R_{\geq 0}$ satisfying
$\cF(q) \geq 0$ and $\cF(q) = 0 \Leftrightarrow q = p$, sampling can be written as the
variational problem of finding $q^* = \argmin_{q \in \cP_2} \cF(q)$.

Minimizing $\cF$ via gradient descent in the 2-Wasserstein metric yields the
\emph{Wasserstein gradient flow}~\cite{villani2003topics}. By Otto's calculus~\cite{otto2001geometry},
the steepest descent velocity field in $(\cP_2, \mathcal{W}_2)$ is:
\begin{equation}
  V_t(x) = -\nabla_x \frac{\delta \cF}{\delta q_t}(x),
  \label{eq:wgf}
\end{equation}
where $\delta \cF / \delta q$ denotes the first variation of $\cF$ with respect to $q$.
The distribution $q_t$ evolves according to the continuity equation
$\partial_t q_t + \nabla \cdot (q_t V_t) = 0$, ensuring $q_t$ remains a valid probability
measure for all $t$. At the particle level, this admits the characteristic ODE interpretation:
any $x_0 \sim q_0$ transported by the flow satisfies
\begin{equation}
  \dot{x}_t = V_t(x_t), \qquad x_t \sim q_t,
  \label{eq:particle_ode}
\end{equation}
where $V_t$ depends on the current marginal $q_t$, coupling all particles through the shared
distribution. Euler-discretizing with step size $h$,
\begin{equation}
  x_{t+h} = x_t + h\,V_t(x_t),
  \label{eq:euler_step}
\end{equation}
gives the corresponding one-step particle update. If the flow converges to the target
equilibrium, then the distribution of transported particles approaches $p$.

The seminal result of \citet{jordan1998variational} established that when
$\cF = \KL(q \| p)$, Eq.~\eqref{eq:wgf} yields the velocity field
$V_t(x) = \nabla \log p(x) - \nabla \log q_t(x)$, whose associated continuity equation
is precisely the Fokker--Planck equation of overdamped Langevin
dynamics (see Appendix~\ref{app:proof_unified} for the general derivation and this
special case). This identifies classical Langevin sampling as the $\KL(q \| p)$ Wasserstein
gradient flow and naturally raises the question of what drifting fields arise from other choices
of $\cF$.

\subsection{The Drifting Model and Training-Time Amortization}
\label{sec:drifting_background}

\citet{deng2026drifting} propose \emph{Drifting Models}, which avoid inference-time
integration of the particle ODE~\eqref{eq:particle_ode} by amortizing it into the training
process. The key observation is that an iterative optimizer (e.g.\ SGD) already generates a
discrete sequence of generators $f_{\theta_0}, f_{\theta_1}, \ldots$, whose pushforwards
$q_i = (f_{\theta_i})_{\#} p_\epsilon$ trace a trajectory through $\cP_2$.
If one can bias each training step so that
\begin{equation}
  f_{\theta_{i+1}}(\epsilon) \approx f_{\theta_i}(\epsilon) + \mathbf{V}\bigl(f_{\theta_i}(\epsilon)\bigr)
  \quad \text{for typical } \epsilon \sim p_\epsilon,
  \label{eq:amortized_euler}
\end{equation}
then the training trajectory $q_0 \to q_1 \to \cdots$ implements a discrete Euler integration
of the WGF~\eqref{eq:euler_step} at the generator level---without any inference-time iteration.
This is enforced by the stop-gradient training objective:
\begin{equation}
  \cL(\theta) = \E_\epsilon \bigl[ \norm{f_\theta(\epsilon) - \sg\bigl(f_\theta(\epsilon) + \mathbf{V}(f_\theta(\epsilon))\bigr)}^2 \bigr],
  \label{eq:drifting_loss}
\end{equation}
where $\sg(\cdot)$ denotes the stop-gradient operator. Minimizing $\cL$ encourages
the generator outputs to match the one-step transport target
$f_\theta(\epsilon) + \mathbf{V}(f_\theta(\epsilon))$, i.e., it amortizes an Euler-type
transport update into parameter learning.
At convergence, $\cL(\theta) \approx 0$ implies $\mathbf{V}(f_\theta(\epsilon)) \approx 0$,
which (under mild conditions) implies $(f_\theta)_{\#} p_\epsilon \approx p$.
At inference, a single forward pass $x = f_\theta(\epsilon)$ produces a sample from $p$.

The drifting field in~\citep{deng2026drifting} is implemented through an attraction--repulsion
kernel: generated samples $x$ are attracted by positive samples $y^+ \sim p$ and repelled by
other generated samples $y^- \sim q$, using a Laplace kernel
$k(x,y) = \exp(-\|x-y\|/\tau)$. This formulation is effective in the data-based setting, but it
requires direct access to samples from $p$ and therefore does not transfer directly to the
data-free case.

\subsection{The Data-Free Setting}
\label{sec:data_free}

In this work we consider the \emph{data-free} setting: we assume access only to
$\nabla E(x) = -\nabla \log p(x)$ for any $x \in \R^d$, where $p(x) = \exp(-E(x)) / Z$ for an
unknown partition function $Z$. We cannot draw samples from $p$, and we cannot evaluate $p(x)$
in normalized form. The generator $f_\theta$ maps noise
$\epsilon \sim p_\epsilon$ (typically Gaussian) to samples in $\R^d$. The induced
pushforward distribution $q = (f_\theta)_{\#} p_\epsilon$ can be sampled from
by forward passes of $f_\theta$, but its density $q(x)$ and score $\nabla \log q(x)$ are
generally intractable unless $f_\theta$ is constrained to be an invertible architecture or is
paired with an auxiliary density model.

This mismatch defines the central implementation challenge of the data-free setting. The
target-side quantity $\nabla \log p(x)$ is available through the energy gradient, but the
$q$-side quantities required by Wasserstein gradient-flow drifts must be recovered from generated
particles alone. The remainder of the paper develops a unified theory for these drifts and then
turns that theory into practical data-free implementations.

% ============================================================
\section{Unified Framework}
\label{sec:unified}
% ============================================================

\subsection{The Unified Drifting Field for f-Divergences}
\label{sec:unified_formula}

We derive the Wasserstein gradient flow for a broad class of f-divergences. Recall that for a convex function
$f: (0,\infty) \to \R$ with $f(1) = 0$, the f-divergence from $p$ to $q$
is~\cite{csiszar1964informationstheoretische}:
$D_f(p \| q) = \int q(x) f(p(x)/q(x))\, dx$.

\begin{theorem}
  \label{thm:unified}
  Let $\cF(q) = D_f(p \| q)$ for a twice-differentiable strictly convex function
  $f:(0,\infty)\to\R$ with $f(1)=0$. Assume that $p$ and $q$ admit strictly positive
  $C^1$ densities on $\R^d$, that $r(x)=p(x)/q(x)$ is well-defined and $C^1$, and that
  the first variation of $\cF$ is well-defined. Then the Wasserstein gradient-flow
  velocity field associated with $\cF$ is
  \begin{equation}
    \mathbf{V}^{D_f}(x) = w(r(x)) \cdot \beta(x),
    \label{eq:unified}
  \end{equation}
  where
  \[
    \beta(x)=\nabla\log r(x)=\nabla\log p(x)-\nabla\log q(x),
    \qquad
    w(r)=r^2 f''(r).
  \]
\end{theorem}
The proof is deferred to Appendix~\ref{app:proof_unified}.

Theorem~\ref{thm:unified} isolates the central structural point of the paper.
The score difference $\beta(x) = \nabla \log r(x) = \nabla \log p(x) - \nabla \log q(x)$
represents the \emph{direction} in which $q$ should move to become more similar to $p$ at
location $x$; it is the universal gradient signal shared by all f-divergences covered by the
theorem. The weight function $w(r)$ then modulates the \emph{strength} of this signal according
to the local density ratio $r = p/q$. In well-covered regions ($r \approx 1$), $w(r)$ remains
close to $w(1)$; the behavior in under-covered regions ($r \gg 1$) is what distinguishes one
divergence from another, as we formalize in Section~\ref{sec:mode_coverage}. Thus the theorem
separates a universal correction direction from an objective-dependent spatial reweighting.
Specific instances of Eq.~\eqref{eq:unified} for the major f-divergences are collected in
Table~\ref{tab:divergences}; in particular, we will later use the Tsallis subfamily
$f_\alpha(r)=(r^\alpha-1)/(\alpha-1)$ for $\alpha>0$, $\alpha\neq1$, whose weight is
$w_\alpha(r)=\alpha r^\alpha$ and whose $\alpha=1$ case is recovered by continuity as Forward KL.

\begin{table}[t]
  \centering
  \caption{Unified drifting fields for the f-divergence family. For each divergence,
    $w(r) = r^2 f''(r)$ gives the weight function and $\beta(x) = \nabla \log r(x)$;
    here $r = p(x)/q(x)$.}
  \label{tab:divergences}
  \smallskip
  \begin{tabular}{lccc}
    \toprule
    $\cF(q)$ & $f(r)$ & $w(r) = r^2 f''(r)$ & $\mathbf{V}$ \\
    \midrule
    $\KL(q \| p)$  & $-\log r$     & $1$         & $\beta$ \\
    $\KL(p \| q)$  & $r \log r$    & $r$         & $(p/q)\cdot\beta$ \\
    $\chi^2(p \| q)$ & $(r-1)^2$   & $2r^2$      & $2(p/q)^2 \cdot \beta$ \\
    Tsallis $T_\alpha$ & $(r^\alpha-1)/(\alpha-1)$ & $\alpha r^\alpha$ & $\alpha(p/q)^\alpha\cdot\beta$ \\
    \bottomrule
  \end{tabular}
\end{table}

\subsection{Log-Variance Divergence}
\label{sec:lv}

The Log-Variance (LV) divergence lies outside the f-divergence family, yet admits an analogous
Wasserstein variational derivation. Given an unnormalized target $\tilde{p}(x) = Z \cdot p(x) = \exp(-E(x))$,
define the \emph{unnormalized density ratio}
\begin{equation}
  \tilde{r}(x) \coloneqq \frac{\tilde{p}(x)}{q(x)} = Z \cdot r(x),
  \label{eq:rtilde}
\end{equation}
so that $\log\tilde{r}(x) = \log\tilde{p}(x) - \log q(x) = \log r(x) + \log Z$.
Let $\nu$ be a reference probability distribution over $\R^d$ that serves as the integration
measure. The LV functional is:
\begin{equation}
  \cL(q) = \Var_\nu\!\bigl[\log\tilde{r}\bigr] = \E_\nu[\log^2\!\tilde{r}] - \bigl(\E_\nu[\log\tilde{r}]\bigr)^2.
  \label{eq:lv}
\end{equation}

A key structural property is the \emph{Z-Annihilator}: since $\log\tilde{r}(x) = \log r(x) + \log Z$,
the centered quantity $\log\tilde{r}(x) - \E_\nu[\log\tilde{r}]$ equals $\log r(x) - \E_\nu[\log r]$,
so the unknown additive constant $\log Z$ cancels exactly. This removes the
partition-function obstruction from the LV drift. However, the resulting field is only as
tractable as the remaining quantities under the chosen reference measure $\nu$; in particular,
expectations under $\nu=p$ still require additional approximation machinery in the data-free
setting.
We derive three cases according to the choice of $\nu$ (full derivations in Appendix~\ref{app:lv}):

\textbf{Case 1 ($\nu = q$).} The first variation involves a \emph{double coupling} (both the
integration measure and the integrand depend on $q$), yielding
\begin{equation}
  \mathbf{V}^{\nu=q}(x) = 2\bigl[1 - (\log\tilde{r}(x) - \E_q[\log\tilde{r}])\bigr] \cdot \beta(x).
  \label{eq:lv_case1}
\end{equation}
Write
\[
  z_q(x) \coloneqq \log\tilde r(x)-\E_q[\log\tilde r].
\]
Then the prefactor is $2(1-z_q(x))$. When $z_q(x)>1$, the coefficient becomes negative and
reverses the baseline direction $\beta(x)$. Thus the fully coupled choice $\nu=q$ does not obey
a monotone ``more under-covered implies stronger repair'' principle. Instead, it regularizes
fluctuations of the log-ratio over the \emph{current} $q$-support, so regions with very large
centered log-ratio can be attenuated or even pushed against $\beta$. From the standpoint of mode
coverage, this makes Case~1 the least robust of the three choices despite its analytical
tractability.

\textbf{Case 2 ($\nu = p$).} With $\nu$ fixed to the target, the first variation is uncoupled,
giving
\begin{equation}
  \mathbf{V}^{\nu=p}(x) = 2\,\frac{p(x)}{q(x)} \cdot \bigl[1 + (\log\tilde{r}(x) - \E_p[\log\tilde{r}])\bigr] \cdot \beta(x).
  \label{eq:lv_case2}
\end{equation}
This is the target-weighted LV objective. Unlike
Case~1, the sign is not reversed merely because the centered log-ratio is large; the only
sign-sensitive term is the bracket
\[
  1 + \bigl(\log\tilde r(x)-\E_p[\log\tilde r]\bigr).
\]
Accordingly, on any region where this bracket is nonnegative, the drift remains aligned with
$\beta$ and is amplified by the target-side weight $p/q$. This makes $\nu=p$ the most directly
target-weighted LV objective for mode repair among the three cases. Its drawback is computational
rather than structural: implementing it in the data-free setting requires approximating
expectations under $p$, which
introduces importance-sampling error.

\textbf{Case 3 ($\nu$ independent of $q$).} When $\nu$ is a fixed distribution whose density
does not depend on the current $q$, the product rule applied to $\nu(x)/q(x)$ introduces an
additional shape-correction term:
\begin{equation}
  \mathbf{V}^{\nu}(x) = 2\,\frac{\nu(x)}{q(x)}\Bigl[\beta(x) + \underbrace{\bigl(\log\tilde{r}(x) - \E_\nu[\log\tilde{r}]\bigr)\,\nabla\log\tfrac{\nu(x)}{q(x)}}_{\gamma(x):\;\text{shape-correction}}\Bigr].
  \label{eq:lv_case3}
\end{equation}
When $\nu=q_{\mathrm{sg}}$ is set to the current $q$ with the sampling measure frozen by
stop-gradient, $\gamma(x) = 0$ and the field reduces to $2\beta(x)$, recovering the Reverse-KL
drift direction with a factor-of-two gain in magnitude.
More generally, Case~3 is a fixed-reference design: its behavior depends on how $\nu$ is chosen,
because the shape-correction term $\gamma(x)$ can either help or hinder target-side repair.
In neural sampler training, replay-buffer techniques~\citep{akhound2024iterated} improve sample efficiency by reusing
historical samples across iterations. Setting $\nu$ as the replay buffer, a mixture of past
generator distributions, instantiates Case~3. Such a reference can retain broader historical
support than the current $q$, which can mitigate the support-myopia of Case~1, but this
flexibility comes with additional estimation error for $\nu$ and its score.

In practice, we adopt an LV-inspired surrogate built from centered log-ratio gating in
Section~\ref{sec:objective_impl}.

\subsection{Connections to the Drifting Model and SVGD}
\label{sec:recovery}
\label{sec:svgd_connection}

Setting $\cF = \KL(q \| p)$ in Theorem~\ref{thm:unified} gives $w(r) = 1$ and
$\mathbf{V}(x) = \beta(x) = \nabla \log r(x)$.
Define the \emph{mean-shift vector}~\cite{cheng1995mean}
$m_q(x) \coloneqq \sum_j W_j(x_j - x)$,
where $W_j = \mathrm{softmax}_j(-\|x - x_j\|^2/\tau)$ are the Gaussian kernel weights.
Under the Gaussian kernel $k(x,y) = \exp(-\|x-y\|^2/\tau)$, the KDE score satisfies
$\nabla \log \hat{q}(x) = \frac{2}{\tau}\,m_q(x)$ exactly (Appendix~\ref{app:recovery}), so
$\mathbf{V}(x) \approx \frac{2}{\tau}(m_p(x) - m_q(x))$.
Thus our framework recovers the Reverse-KL \emph{score-difference field} underlying kernelized
drifting methods. The original Drifting Model~\citep{deng2026drifting} uses a Laplace
kernel with unnormalized displacement vectors, so it should be regarded as a related
heuristic realization of the same Reverse-KL directional principle rather than as an
exact identity with the Gaussian KDE formula. A full derivation and kernel comparison
appear in Appendix~\ref{app:recovery}.

Beyond this connection to kernelized drifting methods, the same field is also closely related to
particle-based inference methods.
Our drift field $\beta(x) = \nabla\log r(x)$ is also structurally related to the optimal
perturbation direction of SVGD~\cite{liu2016stein},
$\phi^*_{q,p}(x) = \mathbb{E}_{y\sim q}[k(y,x)\,\beta(y)]$:
SVGD can be viewed as a kernel-smoothed version of our drift field, where each particle's
update is a weighted average of $\beta$ over the current distribution rather than the
pointwise value $\beta(x)$ itself.
See Appendix~\ref{app:svgd} for the full lemma, proof, corollary, and comparison table.

% ============================================================
\section{Theoretical Analysis}
\label{sec:mode_coverage}
% ============================================================

This section studies the first-order effect of one Euler step of
$\mathbf{V}^{D_f}=w(r)\beta$ on under-covered regions through the soft under-coverage functional
$U_{\delta,\varepsilon}(q)$.  The main result is a first-order variation formula for
$U_{\delta,\varepsilon}$ together with a sufficient criterion for one-step improvement.  Proofs of
all statements in this section are deferred to Appendix~\ref{app:mode_coverage_proofs}.

\subsection{Definitions}
\label{sec:coverage_defs}

We first introduce the geometric region and soft functional used to quantify under-coverage.

\begin{definition}
  For fixed thresholds $0 < \varepsilon \ll \delta$, let
  \[
    \Omega_{\delta,\varepsilon}(q) = \{x \in \R^d : p(x) \geq \delta,\; q(x) \leq \varepsilon\}
  \]
  denote the \emph{under-covered region}: points where the target has significant mass but
  the current sampler density does not. Introduce the \emph{under-coverage score}
  \[
    D_q(x) \coloneqq \log\frac{p(x)}{q(x)} = \log r(x).
  \]
  Large positive values of $D_q(x)$ indicate strong local under-coverage.
  Define the \emph{soft under-coverage mass}
  \begin{equation}
    U_{\delta,\varepsilon}(q)
    \coloneqq
    \int_{\R^d} p(x)\,\mathbbm{1}_{\{p(x)\ge \delta\}}\,\bigl[\varepsilon-q(x)\bigr]_+\,dx,
    \label{eq:soft_undercoverage}
  \end{equation}
  where $[a]_+ \coloneqq \max\{a,0\}$. This functional measures the target mass of
  under-covered regions, weighted by the amount of density deficit below the threshold
  $\varepsilon$.
  If $\Omega_{\delta,\varepsilon}(q)=\emptyset$, then $q(x) > \varepsilon$ wherever
  $p(x) \geq \delta$; coverage is declared sufficient at level $(\delta,\varepsilon)$.
  The set $\Omega_{\delta,\varepsilon}(q)$ is therefore a hard geometric notion of
  under-coverage, while $U_{\delta,\varepsilon}(q)$ is its softened, target-weighted summary.
\end{definition}

\subsection{One-Step Regional Repair}
\label{sec:regional_repair}

We now analyze the first-order response of the soft under-coverage mass under the Euler update in \ref{eq:euler_step}.
\begin{proposition}
  \label{prop:soft_undercoverage}
  Let $q_{t+h}$ be obtained from $q_t$ by one Euler step with drift field $\mathbf{V}$, and
  assume that $q_t$ and $\mathbf{V}$ are $C^1$ on $\{x:p(x)\ge \delta\}$ and that the boundary
  set $\{x:\ p(x)\ge \delta,\ q_t(x)=\varepsilon\}$ has Lebesgue measure zero. Then, formally as
  $h\to 0$,
  \begin{equation}
    U_{\delta,\varepsilon}(q_{t+h})
    =
    U_{\delta,\varepsilon}(q_t)
    - h\,G_{\mathbf{V}}(t;\delta,\varepsilon)
    + o(h).
    \label{eq:soft_undercoverage_decay}
  \end{equation}
  Here
  \begin{equation}
    G_{\mathbf{V}}(t;\delta,\varepsilon)
    \coloneqq
    \int_{\Omega_{\delta,\varepsilon}(q_t)}
    p(x)\,\bigl[-\nabla\cdot(q_t\mathbf{V})(x)\bigr]\,dx
    \label{eq:regional_repair_score}
  \end{equation}
  is the one-step regional repair score induced by the drift field $\mathbf{V}$.
  In particular, if
  $G_{\mathbf{V}}(t;\delta,\varepsilon)>0$,
  then $ U_{\delta,\varepsilon}(q_{t+h})<U_{\delta,\varepsilon}(q_t)$
  for all sufficiently small $h>0$.
\end{proposition}

The sign condition $G_{\mathbf{V}}(t;\delta,\varepsilon)>0$ is integral rather than pointwise:
it depends on the $p$-weighted average of $-\nabla\cdot(q_t\mathbf{V})$ over
$\Omega_{\delta,\varepsilon}(q_t)$ and does not require pointwise positivity on the whole region.
It is a sufficient criterion for first-order one-step improvement of
$U_{\delta,\varepsilon}$, not a global convergence guarantee. For f-divergence drifts, this
quantity admits a simple compression--elasticity form.

\begin{proposition}
  \label{prop:regional_repair_fdiv}
  Let $\mathbf{V}^{D_f}(x)=w(r(x))\,\beta(x)$ be the drifting field of
  Theorem~\ref{thm:unified}, and assume $q_t(x)>0$ and $w(r(x))>0$ for almost every
  $x\in\Omega_{\delta,\varepsilon}(q_t)$.  On the set where $\|\beta(x)\|>0$, define the
  geometric compression index
  \begin{equation}
    \kappa_t(x)
    \coloneqq
    \frac{-\nabla\cdot(q_t\beta)(x)}{q_t(x)\,\|\beta(x)\|^2},
    \label{eq:kappa_t}
  \end{equation}
  and the divergence elasticity
  \begin{equation}
    \eta_f(r)
    \coloneqq
    \frac{r\,w'(r)}{w(r)}.
    \label{eq:eta_f}
  \end{equation}
  Then
  \begin{equation}
    G_f(t;\delta,\varepsilon)
    =
    \int_{\Omega_{\delta,\varepsilon}(q_t)}
    p(x)\,q_t(x)\,w(r(x))\,\|\beta(x)\|^2
    \bigl[\kappa_t(x)-\eta_f(r(x))\bigr]\,dx.
    \label{eq:regional_repair_elasticity}
  \end{equation}
  where the integrand is understood as zero on $\{x:\|\beta(x)\|=0\}$.
  Moreover, if $\kappa_t(x)\ge \eta_f(r(x)) \text{ for a.e. }x\in\Omega_{\delta,\varepsilon}(q_t)\cap\{\|\beta(x)\|>0\},$
  and the inequality is strict on a subset of $\Omega_{\delta,\varepsilon}(q_t)$ of positive
  measure on which $p\,q_t\,w(r)\,\|\beta\|^2>0$, then $G_f(t;\delta,\varepsilon)>0.$
\end{proposition}

\begin{table}[t]
  \centering
  \caption{Compression--elasticity summary for the standard f-divergence drifts. Here
  $\Omega_t\coloneqq \Omega_{\delta,\varepsilon}(q_t), \kappa_t(x)
    \coloneqq
    \frac{-\nabla\cdot(q_t\beta)(x)}{q_t(x)\,\|\beta(x)\|^2}$, and the corresponding weights
  $w(r)$ are listed in Table~\ref{tab:divergences}.}
  \label{tab:regional_repair_examples}
  \small
  \begin{tabular}{lccc}
    \toprule
    Divergence & $\eta_f(r)$ & Sufficient condition & regional integrand \\
    \midrule
    $\KL(q\|p)$
    & $0$
    & $\kappa_t\ge 0$ &  $p\,q_t\,\|\beta\|^2\,\kappa_t$ \\
    $\KL(p\|q)$
    & $1$
    & $\kappa_t\ge 1$ &
      $p^2\,\|\beta\|^2\,(\kappa_t-1)$ \\
    $\chi^2(p\|q)$
    & $2$
    & $\kappa_t\ge 2$ &
      $\dfrac{2p^3}{q_t}\,\|\beta\|^2\,(\kappa_t-2)$ \\
    Tsallis $T_\alpha$
    & $\alpha$
    & $\kappa_t\ge \alpha$ &
      $p\,q_t\,\alpha r^\alpha\,\|\beta\|^2\,(\kappa_t-\alpha)$ \\
    \bottomrule
  \end{tabular}
\end{table}

In every row of Table~\ref{tab:regional_repair_examples}, the displayed threshold is only the
almost-everywhere part of the sufficient condition; to conclude $G_f(t;\delta,\varepsilon)>0$,
the inequality must also be strict on a subset of $\Omega_t$ of positive measure where the listed
regional integrand is positive.

\paragraph{Geometry Intuition.}
For any fixed Lipschitz region $A$, the continuity equation gives
\[
  \frac{d}{dt}\int_A q_t(x)\,dx
  =
  -\int_{\partial A} q_t(x)\,\mathbf{V}(x)\cdot n(x)\,dS(x),
\]
so positive regional repair corresponds to positive net inward flux into the under-covered region
on average.  Proposition~\ref{prop:regional_repair_fdiv} shows how the divergence choice changes
this flux by comparing the geometric compression index $\kappa_t$ with the divergence elasticity
$\eta_f$.  Table~\ref{tab:regional_repair_examples} gives the resulting thresholds
$0,1,2,\alpha$ for Reverse KL, Forward KL, $\chi^2$, and Tsallis, respectively. This provides a
concrete criterion for comparing how different f-divergence drifts respond to the same
under-covered region.  Section~\ref{sec:exp_demos} provides a direct empirical illustration:
Figure~\ref{fig:exp_unified_drift} (bottom row) visualises $\kappa_t$, the gap
$\kappa_t-\eta_f$, and $G_f$ for Reverse KL, Forward KL, and $\chi^2$ on a converged GMM-8
checkpoint, while Figure~\ref{fig:exp_regional_repair} confirms $G_{\mathbf{V}}>0$ via a
frozen one-step probe on an early-state proxy.

\subsection{Fixed-Point Analysis}
\label{sec:fixed_point}

A final structural question is whether $\mathbf{V}_\cF(x)=\mathbf{0}$ for all $x$ forces
$p=q$, or whether degenerate fixed points can occur. For the present f-divergence drifts, the
explicit factorization $\mathbf{V}=w(r)\beta$ makes the answer immediate. 

\begin{proposition}
  \label{prop:fixed_point}
  Suppose $w(r)>0$ for all $r>0$, and assume that $p$ and $q$ are strictly positive densities on a
  connected domain and that $\beta(x)=\nabla\log(p/q)$ is well-defined. If
  $\mathbf{V}(x)=w(r(x))\beta(x)=0$ for all $x$, then $p=q$.
\end{proposition}

For the LV divergence, this direct reduction is unavailable. Its drifting field contains
additional multiplicative and additive factors, so $\mathbf{V}=0$ can occur either because
$\beta=0$ or because the prefactor vanishes, and these mechanisms may coexist on different parts
of the domain. The corresponding degenerate equilibria are summarized in
Appendix~\ref{app:fixed_points}.

\subsection{Interpreting Mode-Seeking and Mode-Covering}
\label{sec:mode_bias_interpretation}

The classical slogan that Reverse KL is mode-seeking while Forward KL is mode-covering originates
from a different setting: static divergence minimization over a restricted parametric family. For
example, when a single Gaussian is used to approximate a multimodal target, $\KL(q\|p)$ is
$q$-weighted, so modes that receive essentially no mass under the current approximation are barely
``seen'' by the objective; by contrast, $\KL(p\|q)$ is $p$-weighted, so any target mode with
non-negligible $p$-mass but insufficient $q$-coverage incurs a strong penalty. In that static
approximation setting, the slogan is therefore a statement about how the two KL objectives weight
missing modes inside a restricted family.
In the present setting, the more relevant distinction is the transient repair dynamics induced by
the drift, together with the actual sample coverage achieved during training.
More precisely, this reinterpretation concerns the present Wasserstein-gradient-flow setting and
its induced one-step repair dynamics; it is not intended as a blanket statement about all
optimization settings involving these divergences.

For the standard f-divergence drifts considered here, Proposition~\ref{prop:fixed_point} shows
that the continuum fixed point is always $p=q$ whenever $w(r)>0$. Accordingly, in our setting,
the same slogan should not be interpreted as a statement about different asymptotic targets. All
standard f-divergence drifts share the same local correction direction $\beta(x)=\nabla\log\frac{p(x)}{q(x)}$,
and differ only through the spatially varying prefactor $w(r(x))$.
At the particle level, this prefactor acts as a location-dependent effective step size:
\[
  x_{t+h}=x_t + h\,w(r(x_t))\,\beta(x_t).
\]
The substantive difference is that $w(r(x))$ varies across space through the local
ratio $r(x)=p(x)/q_t(x)$, so different divergences redistribute correction effort differently
across the current particle cloud.
The regional-response identity \eqref{eq:regional_repair_elasticity} makes this precise:
\[
  G_f(t;\delta,\varepsilon)
  =
  \int_{\Omega_t}
  p(x)\,q_t(x)\,w(r(x))\,\|\beta(x)\|^2
  \bigl[\kappa_t(x)-\eta_f(r(x))\bigr]\,dx.
\]
Here $w(r)$ determines how strongly severely under-covered regions enter the repair score, while
$\eta_f(r)=r w'(r)/w(r)$ determines the geometric threshold required for positive regional
repair. In this sense, the differences among Reverse KL, Forward KL, $\chi^2$, and Tsallis are
best understood as differences in \emph{transient regional repair bias}, rather than differences
in the final target of the flow. In practice, stronger reweighting is beneficial only when the
current sampler and density estimator already provide sufficiently informative samples in the
relevant under-covered regions; otherwise the additional prefactor may amplify ratio- and
score-estimation error instead of improving repair.

% ============================================================
\section{Data-Free Implementation}
\label{sec:implementation}
% ============================================================

In the data-free setting, implementing the standard f-divergence drifts listed in
Table~\ref{tab:divergences} requires $\nabla \log p(x) = -\nabla E(x)$ together with the
$q$-side quantities $\log q(x)$ and $\nabla \log q(x)$. We therefore view data-free
implementation for this multiplicative family as the problem of recovering the pair
$(\log q, \nabla \log q)$ along the current particle cloud. In the present version of the paper,
our primary implementation is mini-batch KDE; the same framework also naturally admits a
complementary normalizing-flow route, which we describe here as a model-based extension path.

\subsection{Kernel Density Estimation}
\label{sec:kde_impl}

The KDE implementation estimates both density and score directly from the current particles via
mini-batch kernel density estimation.
Given generated samples $\{x_j\}_{j=1}^N$, define
\[
  \hat q(x) = \frac{1}{N}\sum_{j=1}^N k(x, x_j).
\]
Then the score satisfies
\begin{equation}
  \nabla \log \hat{q}(x) = \frac{\sum_j \nabla_x k(x, x_j)}{\sum_j k(x, x_j)}.
  \label{eq:kde_score_general}
\end{equation}
For the specific kernels, we have
\begin{align}
  \text{RBF: } k(x,y) = e^{-\|x-y\|^2/\tau}:&\quad
    \nabla \log \hat{q}(x) = \frac{2}{\tau}\sum_j W_{ij}^{\mathrm{rbf}}\,(x_j - x),
    \label{eq:kde_rbf} \\
  \text{Laplace: } k(x,y) = e^{-\|x-y\|/\tau}:&\quad
    \nabla \log \hat{q}(x) \propto \sum_j W_{ij}^{\mathrm{lap}}\,\frac{(x_j - x)}{{{\|x_j - x\|}}},
    \label{eq:kde_laplace}
\end{align}
where $W_{ij}^{\mathrm{rbf}} = {\mathrm{softmax}_j(-\|x-x_j\|^2/\tau)}$ and
$W_{ij}^{\mathrm{lap}} = {\mathrm{softmax}_j(-\|x-x_j\|/\tau)}$.
The RBF estimate is \emph{exact} (the weighted mean-shift vector equals $\frac{\tau}{2}\,\nabla\log\hat{q}$),
while the Laplace estimate yields a unit-displacement-weighted analogue of $\nabla\log\hat{q}$
rather than an exact score identity.
The original Drifting Model~\citep{deng2026drifting} adopts a structurally similar form to
\eqref{eq:kde_laplace} but uses unnormalized displacements $(x_j - x)$ rather than unit vectors (see Appendix~\ref{app:recovery} for details).

This interpretation is consistent with the classical mean-shift literature, where updates move a
point toward a kernel-weighted local average~\citep{cheng1995mean}. Accordingly, the
Drifting-Model-style Laplace update based on unnormalized displacements $(x_j-x)$ can be viewed
as a reasonable mean-shift transport heuristic even though it is not the exact gradient of a
standard Laplace KDE. Its practical appeal is that it yields a direct displacement field and is
typically more numerically stable for the stop-gradient one-step training objective.

\paragraph{Optional Sinkhorn stabilization.}
For the score estimate in~\eqref{eq:kde_rbf}, plain row-softmax weights become overly local when
$\tau$ is small, which weakens inter-mode repulsion. In practice we therefore optionally replace
them by a doubly-stochastic Sinkhorn coupling before forming the mean-shift vector. This keeps
repulsive interactions more global while preserving the same particle-to-particle computational
template. The log-space iteration, its relation to the double-softmax normalization of the
original Drifting Model, and the small-bandwidth collapse analysis are given in
Appendix~\ref{app:sinkhorn}.

\subsection{Normalizing-Flow Density Models}
\label{sec:nf_impl}

The normalizing-flow route uses an NF density model to provide the same pair
$(\log q,\nabla\log q)$. In the exact-flow instantiation, one restricts $f_\theta$ itself to an invertible NF
architecture (e.g., RealNVP~\cite{dinh2016density}), whose change-of-variables formula gives
\begin{equation}
  \log q(x) = \log p_\epsilon(f_\theta^{-1}(x)) - \log \left|\det J_{f_\theta^{-1}}(x)\right|.
  \label{eq:nf_density}
\end{equation}
The score $\nabla \log q(x)$ then follows from automatic differentiation of~\eqref{eq:nf_density}
without any kernel approximation, and the unnormalized ratio
$\tilde r(x) = \exp(-E(x))/q(x)=Zr(x)$ is also directly available. For the standard weights in
Table~\ref{tab:divergences}, which are homogeneous in $r$, replacing $r$ by $\tilde r$ changes
$w(r)$ only by a global multiplicative factor, which can be absorbed into the drift scale or step
size. This is distinct from the Z-Annihilator identity of Section~\ref{sec:lv}, which applies to
centered log-ratios in the LV family.
More generally, one may fit an auxiliary NF density model $\hat q_\phi$ to the current particles;
then $\log \hat q_\phi$ and $\nabla\log \hat q_\phi$ are still available analytically, but the
resulting drifting field carries the NF learning error of $\phi$. The practical trade-off is
standard: KDE is nonparametric but bandwidth-sensitive and $O(N^2)$, whereas NF evaluation is
analytic after fitting but inherits density-model representation and optimization error. Classical
KDE bias--variance bounds and the corresponding score-optimal bandwidth scaling are standard; see
\citet{silverman1986density}, \citet{wand1995kernel}, \citet{tsybakov2009introduction}, and
\citet{gine2002rates}.
In the current paper, this NF route is included as a framework-supported implementation path; our
experiments focus on the KDE instantiation.

\subsection{Objective-Specific Implementation}
\label{sec:objective_impl}

The objectives considered in this section all fit the same sampler--estimator decomposition. The
mandatory component is a neural data sampler
$x_i = f_\theta(\epsilon_i)$ with $\epsilon_i \sim p_\epsilon$, which produces the particle cloud
on which the drifting field is evaluated. A second component, the density estimator, recovers the
$q$-side quantities from these particles. In the present paper this estimator is either the
sampler-induced KDE of Section~\ref{sec:kde_impl} or an additional NF density model as in
Section~\ref{sec:nf_impl}. Both provide the same interface
\[
  \ell_i \approx \log q(x_i),
  \qquad
  s_i \approx \nabla\log q(x_i).
\]
Accordingly, the continuum correction direction $\beta(x_i)=\nabla\log p(x_i)-\nabla\log q(x_i)$
is estimated by
\[
  \tilde\beta_i \coloneqq a_{\mathrm{attr}}\,\nabla\log p(x_i)-s_i,
\]
which reduces to the direct score-difference estimate when $a_{\mathrm{attr}}=1$. In practice,
$a_{\mathrm{attr}}$ is implemented as an attraction-weight hyperparameter (`attr\_weight') that
controls the strength of the target-side pull and can be tuned dynamically; see
Section~\ref{sec:experiments}.

\paragraph{f-divergence} For Reverse KL, the drift is simply $\mathbf{V}_i = \tilde\beta_i$,
so only the score estimate is required. 
For the remaining standard f-divergence drifts in Table~\ref{tab:divergences}, one forms
\[
  \tilde r_i \coloneqq \exp\bigl(-E(x_i)-\ell_i\bigr) = Z\,r_i,
\]
and then evaluates the raw multiplicative weights $\tilde w_i = w(\tilde r_i)$. Since the
standard weights considered here are homogeneous in $r$, the unknown normalizing constant
$Z$ contributes only a global batch factor. We therefore use a
\emph{batchwise self-normalized reweighting}:
\[
  \bar w_i
  \coloneqq
  \frac{\tilde w_i}{\frac{1}{N}\sum_{j=1}^N \tilde w_j},
  \qquad
  \mathbf{V}_i = \bar w_i\,\tilde\beta_i,
\]
up to an overall drift scale or step-size parameter. This preserves the relative reweighting
induced by the chosen divergence while removing the unknown global scale carried by $Z$.

\paragraph{LV-divergence} For LV divergences, the exact formulas of
Section~\ref{sec:lv} are less convenient in practice: Case~1 can reverse $\beta$, whereas
Cases~2 and~3 require additional target/reference quantities beyond the basic
sampler--estimator interface. We therefore use the following LV-inspired heuristic. Define
\[
  m_i \coloneqq -E(x_i)-\ell_i \approx \log\tilde r(x_i),
\]
and let $\bar m_\nu$ denote an empirical estimate of $\E_\nu[\log\tilde r]$ under the chosen
reference samples; depending on the implementation, $\nu$ may be taken as the current mini-batch,
the current $q$, or a replay-buffer reference. We then set
\[
  \mathbf{V}_i
  =
  2\Bigl(1+\bigl[m_i-\bar m_\nu\bigr]_+\Bigr)
  \tilde\beta_i.
\]
This heuristic preserves the Z-Annihilator, avoids direction reversal, and leaves the baseline
Reverse-KL drift scaled by a factor of two whenever $m_i\le \bar m_\nu$.
To incorporate the target-side correction suggested by Case~2 while remaining data-free, one may
additionally introduce the batch-normalized variant
\[
  \mathbf V_i = 2\,\bar w_i\Bigl(1+\bigl[m_i-\bar m_\nu\bigr]_+\Bigr)\tilde\beta_i.
\]

We present the high level training algorithm in Algorithm~\ref{alg:unified}.

\begin{algorithm}[t]
  \caption{Data-Free Training for a One-Step Neural Sampler}
  \label{alg:unified}
  \begin{algorithmic}[1]
    \REQUIRE Energy function $E(x)$, neural sampler $f_\theta$, prior $p_\epsilon$, batch size $N$, attraction weight $a_{\mathrm{attr}}$
    \REQUIRE Density estimator for $(\log q,\nabla\log q)$ (Sections \ref{sec:kde_impl} and \ref{sec:nf_impl})
    \REQUIRE Objective-specific rule for constructing the particlewise drift $\mathbf V_i$ (Section~\ref{sec:objective_impl})
    \FOR{each training iteration}
      \STATE Sample $\epsilon_1,\ldots,\epsilon_N \sim p_\epsilon$ and set $x_i=f_\theta(\epsilon_i)$
      \STATE Compute $\nabla\log p(x_i)=-\nabla E(x_i)$, estimate $\ell_i \approx \log q(x_i)$ and $s_i \approx \nabla\log q(x_i)$
      \STATE Form $\tilde\beta_i = a_{\mathrm{attr}}\,\nabla\log p(x_i)-s_i$
      \STATE Construct the objective-specific drifting field $\mathbf V_i$
      \STATE Compute $\cL = \dfrac{1}{N}\sum_i \bigl\|x_i - \sg\bigl(x_i + \mathbf{V}_i\bigr)\bigr\|^2$
      \STATE Update $\theta \leftarrow \theta - \eta\,\nabla_\theta \cL$
    \ENDFOR
  \end{algorithmic}
\end{algorithm}

% ============================================================
\section{Experiments}
\label{sec:experiments}
% ============================================================

We organize the experiments into four parts. Section~\ref{sec:exp_setup} describes the
benchmarks, methods, and evaluation protocol. Section~\ref{sec:exp_benchmarks} reports results
on five GMM benchmarks (Table~\ref{tab:exp_main}): GMM-8 and GMM-40 in two dimensions, and
GMM-Manymodes, GMM-2hard-16, and GMM-2hard-32 in higher dimensions.
Section~\ref{sec:exp_demos} provides two illustrative diagnostics: a visual verification of the
unified drift decomposition and a frozen one-step regional-repair probe aligned with
Section~\ref{sec:regional_repair}. Section~\ref{sec:exp_ablations} studies the implementation
choices introduced in Section~\ref{sec:objective_impl} (Table~\ref{tab:exp_ablation_kde}). Full
training details, extended tables, and additional figures are deferred to
Appendix~\ref{app:experiments_additional}.

\subsection{Experimental Setup}
\label{sec:exp_setup}

\paragraph{Benchmarks.}
We evaluate on Gaussian-mixture benchmarks spanning low-dimensional multimodal structure and
higher-dimensional mode-bridging difficulty. The two-dimensional targets, GMM-8 and GMM-40,
provide direct visual access to multimodal coverage and mode occupancy. The eight-dimensional
benchmark GMM-Manymodes tests whether the same behavior extends beyond the plane under
imbalanced mode weights. Two harder bimodal benchmarks, GMM-2hard-16 and GMM-2hard-32, place
well-separated Gaussian modes in $\mathbb{R}^{16}$ and $\mathbb{R}^{32}$ and therefore isolate
the challenge of transporting mass across a large inter-mode gap. For the two-dimensional
targets we report MMD, exact $W_1$, and mode coverage; for the higher-dimensional targets we
replace exact $W_1$ with a Sinkhorn estimate of $W_2$ and again report MMD and mode coverage.
All metrics are evaluated on generated samples against exact reference samples from the analytic
target. Precise benchmark definitions, sample counts, and metric estimators are deferred to
Appendix~\ref{app:exp_benchmarks} and Appendix~\ref{app:exp_metrics}.

\paragraph{Methods.}
We evaluate two objectives from the unified framework: Reverse KL ($w(r)=1$) and the
LV-inspired surrogate (Section~\ref{sec:objective_impl}). Unless otherwise stated, all main
experiments use Laplacian KDE as the density estimator (Section~\ref{sec:kde_impl}); Gaussian
KDE appears only in the ablation study and in selected qualitative comparisons
(Section~\ref{sec:exp_ablations}).

\paragraph{Training protocol.}
All experiments follow Algorithm~\ref{alg:unified} and use the same general one-step training
pipeline. We optimize with Adam and tune only a small set of benchmark-dependent hyperparameters,
primarily the attraction weight $a_{\mathrm{attr}}$ and the KDE bandwidth $\tau$. When a
schedule is used, $a_{\mathrm{attr}}$ is annealed from a conservative initial value to avoid
premature collapse, while $\tau$ is either held fixed or annealed from an inter-mode scale to a
within-mode scale depending on the geometry of the target. All benchmark-specific choices,
including architectures, optimizer settings, sample counts, and exact schedule endpoints, are
listed in Appendix~\ref{app:exp_training}.

\subsection{Main Results on Multimodal Benchmarks}
\label{sec:exp_benchmarks}

Tables~\ref{tab:exp_main} and~\ref{tab:exp_hd} summarize results across all five benchmarks.

\paragraph{Two-dimensional targets.}
On GMM-8, both objectives achieve complete coverage of all 8 modes across 3 seeds. Reverse KL
attains $W_1 = 0.259 \pm 0.009$ and LV achieves $W_1 = 0.270 \pm 0.014$; the difference is
small and within variance across seeds.
On GMM-40, both objectives achieve complete coverage of all 40 modes. Reverse KL attains a
lower $W_1$ ($3.306 \pm 0.293$ versus $4.543 \pm 1.560$ for LV) and better MMD ($0.049$ versus
$0.075$). The larger $W_1$ and higher variance of LV are consistent with the theoretical
picture in which the centering gate amplifies the drift in strongly under-covered regions and may
improve local mode-weight fidelity at the cost of higher variance in the effective weighting.
Across the benchmarks where both objectives remain stable, neither dominates uniformly:
Reverse KL tends to achieve lower transport cost ($W_1$/$W_2$) and MMD, while LV achieves
better mode-weight fidelity (TVD, $\mathrm{KL}_\mathrm{mode}$) on higher-dimensional targets
where under-coverage is more persistent. The exception is GMM-2hard-32, where LV collapses to
a single mode across all seeds, making Reverse KL the only valid choice on that benchmark.
The selection between objectives can therefore be guided by the relative priority of
distributional fit versus long-range transport, with the caveat that the LV centering gate
becomes unreliable when the inter-mode gap is very large.
Figure~\ref{fig:exp_lowdim} shows representative sample clouds for the two Laplacian-KDE main
variants on both 2D targets, with Gaussian-KDE runs included only as qualitative reference.

\begin{table}[t]
  \centering
  \caption{Results on 2D GMM benchmarks with Laplacian KDE.
    All entries report mean{\scriptsize\,$\pm$\,std} over 3 random seeds. $W_1$ is computed by
    exact linear programme; TVD and $\mathrm{KL}_\mathrm{mode}$ measure empirical mode-weight
    discrepancy. Best result per benchmark in \textbf{bold}.}
  \label{tab:exp_main}
  \smallskip
  \begin{tabular}{llcccc}
    \toprule
    Benchmark & Obj.
      & $W_1\!\downarrow$ & MMD$\downarrow$
      & TVD$\downarrow$ & $\mathrm{KL}_\mathrm{mode}\!\downarrow$ \\
    \midrule
    \multirow{2}{*}{GMM-8}
      & RKL
        & $\mathbf{0.259}{\scriptstyle\,\pm0.009}$
        & $\mathbf{0.000}{\scriptstyle\,\pm0.000}$
        & $0.345{\scriptstyle\,\pm0.009}$
        & $0.0006{\scriptstyle\,\pm0.0002}$ \\
      & LV
        & $0.270{\scriptstyle\,\pm0.014}$
        & $\mathbf{0.000}{\scriptstyle\,\pm0.000}$
        & $\mathbf{0.324}{\scriptstyle\,\pm0.013}$
        & $\mathbf{0.0002}{\scriptstyle\,\pm0.0000}$ \\
    \midrule
    \multirow{2}{*}{GMM-40}
      & RKL
        & $\mathbf{3.306}{\scriptstyle\,\pm0.293}$
        & $\mathbf{0.049}{\scriptstyle\,\pm0.013}$
        & $\mathbf{0.248}{\scriptstyle\,\pm0.015}$
        & $\mathbf{0.055}{\scriptstyle\,\pm0.006}$ \\
      & LV
        & $4.543{\scriptstyle\,\pm1.560}$
        & $0.075{\scriptstyle\,\pm0.035}$
        & $0.283{\scriptstyle\,\pm0.006}$
        & $0.067{\scriptstyle\,\pm0.020}$ \\
    \bottomrule
  \end{tabular}
\end{table}

\begin{figure}[t]
  \centering
  \includegraphics[width=\textwidth]{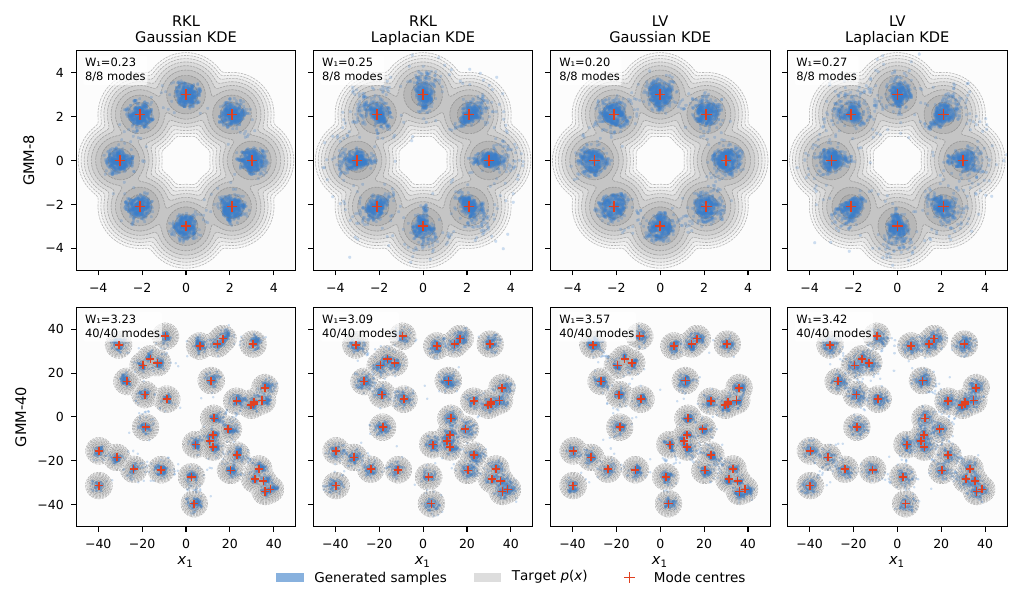}
  \caption{\textbf{Qualitative results on the two-dimensional GMM benchmarks.}
    Each panel overlays 4096 generated samples (blue) on the target log-density contours (grey),
    with mode centres marked by red crosses (\textcolor{red}{$+$}).
    Rows: GMM-8 (top, $K=8$) and GMM-40 (bottom, $K=40$).
    Columns: four qualitative reference variants — Reverse KL with Gaussian KDE, Reverse KL with
    Laplacian KDE, LV with Gaussian KDE, and LV with Laplacian KDE.
    Per-panel annotations report the best-checkpoint $W_1$ and fraction of covered modes.
    The Laplacian-KDE panels correspond to the main 2D benchmark results reported in
    Table~\ref{tab:exp_main}; the Gaussian-KDE panels are included only for qualitative
    comparison. Qualitatively, all four variants cover all modes on both benchmarks (see text).}
  \label{fig:exp_lowdim}
\end{figure}

\begin{table}[t]
  \centering
  \caption{Results on higher-dimensional GMM benchmarks with Laplacian KDE. All entries report mean{\scriptsize\,$\pm$\,std} over 3
    random seeds. $W_2$ is estimated by Sinkhorn; TVD and $\mathrm{KL}_\mathrm{mode}$ measure
    empirical mode-weight discrepancy. Best result per benchmark in \textbf{bold}.}
  \label{tab:exp_hd}
  \smallskip
  \begin{tabular}{llcccc}
    \toprule
    Benchmark & Obj.
      & $W_2\!\downarrow$ & MMD$\downarrow$
      & TVD$\downarrow$ & $\mathrm{KL}_\mathrm{mode}\!\downarrow$ \\
    \midrule
    \multirow{2}{*}{\shortstack[l]{GMM-Many\\($d=8$, $L=8$)}}
      & RKL
        & $\mathbf{1.901}{\scriptstyle\,\pm0.020}$
        & $0.174{\scriptstyle\,\pm0.035}$
        & $0.488{\scriptstyle\,\pm0.077}$
        & $0.213{\scriptstyle\,\pm0.059}$ \\
      & LV
        & $2.016{\scriptstyle\,\pm0.022}$
        & $\mathbf{0.066}{\scriptstyle\,\pm0.032}$
        & $\mathbf{0.274}{\scriptstyle\,\pm0.074}$
        & $\mathbf{0.052}{\scriptstyle\,\pm0.027}$ \\
    \midrule
    \multirow{2}{*}{\shortstack[l]{GMM-2hard\\($d=16$)}}
      & RKL
        & $\mathbf{0.281}{\scriptstyle\,\pm0.012}$
        & $0.262{\scriptstyle\,\pm0.082}$
        & $0.391{\scriptstyle\,\pm0.126}$
        & $0.086{\scriptstyle\,\pm0.055}$ \\
      & LV
        & $0.392{\scriptstyle\,\pm0.013}$
        & $\mathbf{0.215}{\scriptstyle\,\pm0.078}$
        & $\mathbf{0.284}{\scriptstyle\,\pm0.130}$
        & $\mathbf{0.050}{\scriptstyle\,\pm0.039}$ \\
    \midrule
    \multirow{2}{*}{\shortstack[l]{GMM-2hard\\($d=32$)}}
      & RKL
        & $\mathbf{0.849}{\scriptstyle\,\pm0.266}$
        & $\mathbf{0.268}{\scriptstyle\,\pm0.004}$
        & $\mathbf{0.357}{\scriptstyle\,\pm0.017}$
        & $\mathbf{0.065}{\scriptstyle\,\pm0.006}$ \\
      & LV
        & ${0.963}{\scriptstyle\,\pm0.365}$
        & ${0.279}{\scriptstyle\,\pm0.009}$
        & ${0.461}{\scriptstyle\,\pm0.023}$
        & ${0.098}{\scriptstyle\,\pm0.013}$ \\
    \bottomrule
  \end{tabular}
\end{table}

\paragraph{Higher-dimensional targets.}
On GMM-Manymodes ($d=8$, 8 geometrically imbalanced modes), both objectives achieve full mode
coverage; results are shown in Table~\ref{tab:exp_hd}. Reverse KL attains a lower $W_2$
($1.901 \pm 0.020$ versus $2.016 \pm 0.022$ for LV), while LV achieves substantially better
MMD ($0.066$ versus $0.174$) and mode-weight fidelity ($\mathrm{KL}_\mathrm{mode} = 0.052$
versus $0.213$), consistent with the picture that the centering gate amplifies corrections in
under-covered regions.
On GMM-2hard-16 ($d=16$, two modes separated by $\|\mu_1 - \mu_2\| = 8$), Laplacian KDE
successfully bridges the mode gap under both objectives. Reverse KL achieves
$W_2 = 0.281 \pm 0.012$ and LV achieves $W_2 = 0.392 \pm 0.013$; LV again shows better
mode-weight fidelity ($\mathrm{KL}_\mathrm{mode} = 0.050 \pm 0.039$ versus $0.086 \pm 0.055$).
On GMM-2hard-32 ($d=32$, $\|\mu_1-\mu_2\|=11.3$), LV collapses to a single mode
(coverage $1/2$) across all seeds, while Reverse KL produces valid samples with
$W_2 = 0.849 \pm 0.266$.
Taken together, these results confirm that the data-free one-step training procedure
requires no architectural changes as dimension increases, provided the density estimator
bandwidth is appropriately calibrated.
Taken together, these results support the same overall picture as the two-dimensional study:
Reverse KL is the more stable choice for transport across large gaps, while LV can improve
mode-weight fidelity when the estimator remains sufficiently informative in under-covered
regions.

\subsection{Illustrative Demos: Unified Drift and One-Step Regional Repair}
\label{sec:exp_demos}

\paragraph{Unified drift visualization.}
To make the unified formula $\mathbf{V}=w(r)\beta$ operationally visible, we freeze a trained
model at convergence on GMM-8 and evaluate three scalar and vector fields on a fine $80\!\times\!80$
grid: the shared correction direction $\beta=\nabla\!\log p - \nabla\!\log\hat{q}_t$, the
divergence-specific weight $w(r)$, and the resulting drift $\mathbf{V}=w(r)\beta$.
Figure~\ref{fig:exp_unified_drift} (top) organises this decomposition for three divergences.
Panel~(a) shows the single shared $\beta$ field, which is identical for all divergences.
Panels~(b)–(c) display the non-trivial weights $w(r)=r$ (Forward KL) and
$w(r)=2(1+[m-\bar{m}]_+)$ (LV surrogate); note that Reverse KL's weight $w\equiv 1$ is trivial
and omitted.  Panels~(d)–(f) show the resulting drift $\mathbf{V}$ for Reverse KL, Forward KL,
and LV on a shared magnitude scale: the direction structure is identical to $\beta$, while the
spatial amplitude pattern differs according to each $w(r)$.

Figure~\ref{fig:exp_unified_drift} (bottom) shows the compression-elasticity decomposition of
Proposition~\ref{prop:regional_repair_fdiv} evaluated at the same convergent checkpoint.
Panel~(a) shows the shared geometric compression index $\kappa_t$; panels~(b)–(d) show the
gap $\kappa_t-\eta_f$ for Reverse KL ($\eta_f=0$, $G_f=0.038>0$), Forward KL
($\eta_f=1$, $G_f=-0.286<0$), and $\chi^2$ ($\eta_f=2$, $G_f\ll 0$).  At convergence $q_t\approx p$, so $\kappa_t$ is near zero in most of the domain background,
though it retains visible structure near mode centres and along inter-mode channels
(panel~a of Figure~\ref{fig:exp_unified_drift}, bottom row, with colour range $\approx\pm 3$).
The $p$-weighted average of $\kappa_t$ over the full domain is small but positive, so only
Reverse KL's zero-elasticity threshold ($\eta_f=0$) admits a marginally positive $G_f$;
Forward KL and $\chi^2$ impose strictly higher thresholds ($\eta_f=1$ and $\eta_f=2$) that the
weighted-average $\kappa_t$ no longer meets, leading to negative repair scores.  This illustrates how the divergence elasticity
$\eta_f$ governs whether a given divergence provides repair guarantees in the
near-convergence regime.

\begin{figure}[t]
  \centering
  \includegraphics[width=\textwidth]{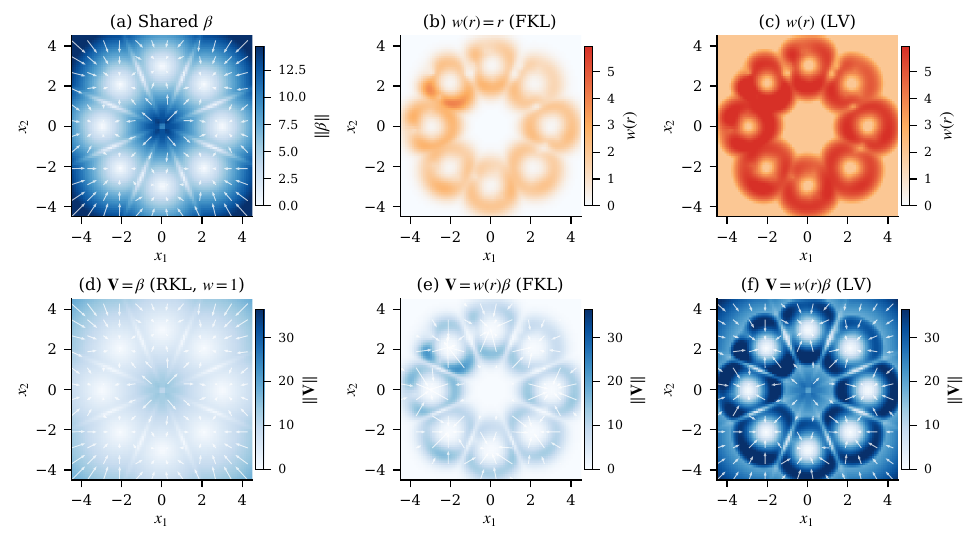}\\[4pt]
  \includegraphics[width=\textwidth]{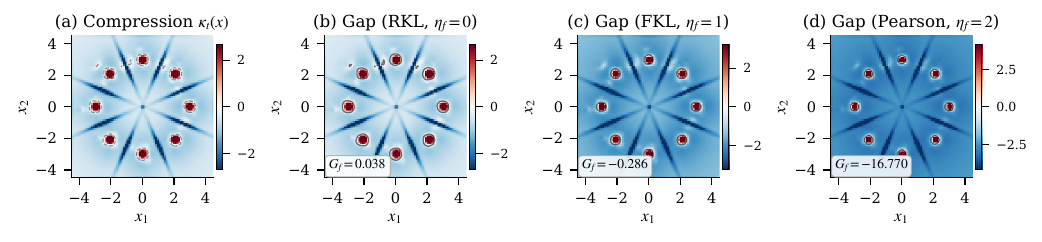}
  \caption{%
    \textbf{Unified drift decomposition on GMM-8 (converged checkpoint).}
    \emph{Top rows (a–f):}
    (a)~shared direction $\beta=\nabla\!\log p - \nabla\!\log\hat{q}_t$, identical for all
    divergences;
    (b)–(c)~divergence-specific weights $w(r)=r$ (Forward KL) and $w(r)=2(1{+}[m{-}\bar{m}]_+)$
    (LV), on a common $w$ scale; Reverse KL's $w\!\equiv\!1$ is trivial and omitted;
    (d)–(f)~resulting drift $\mathbf{V}=w(r)\beta$ for Reverse KL, Forward KL, and LV on a
    shared $\|\mathbf{V}\|$ scale, demonstrating that direction is shared while amplitude
    differs.
    \emph{Bottom row (a–d):} compression-elasticity decomposition at convergence.
    (a)~geometric compression index $\kappa_t$ (shared geometry);
    (b)–(d)~gap $\kappa_t-\eta_f$ with annotated $G_f$ for Reverse KL ($\eta_f\!=\!0$,
    $G_f\!>\!0$), Forward KL ($\eta_f\!=\!1$, $G_f\!<\!0$), and $\chi^2$ ($\eta_f\!=\!2$,
    $G_f\!\ll\!0$); the zero-crossing contour (black line) marks the boundary of the
    repair-active region.}
  \label{fig:exp_unified_drift}
\end{figure}

\paragraph{Frozen one-step regional-repair probe.}
To provide a direct empirical check of Proposition~\ref{prop:regional_repair_fdiv}, we use a set of
$N=2000$ fresh samples drawn from $\mathcal{N}(0,0.64\,I)$ as a proxy for an
early-training snapshot in which $q_t$ is concentrated near the origin and leaves the
peripheral modes of GMM-8 under-covered.  The drift field $\mathbf{V}$ is estimated
on an $80\times 80$ grid using the same Laplacian KDE as in training at bandwidth
$\tau=\tau_{\mathrm{init}}=0.5$, which corresponds to the step-0 KDE state consistent
with the early-training proxy.  The under-coverage region
$\Omega_{\delta,\varepsilon}(q_t)=\{x:p(x)\ge\delta,\;q_t(x)\le\varepsilon\}$
is identified with thresholds $\delta=0.01$ and $\varepsilon=0.01$, chosen
to be consistent with the normalised KDE density $\hat{q}_t=(1/N)\sum_i k(x,x_i)$.

Rather than computing $U_{\delta,\varepsilon}$ exactly, we verify the sufficient condition
$G_{\mathbf{V}}(t;\delta,\varepsilon)>0$ from Proposition~\ref{prop:regional_repair_fdiv}: if this
scalar is positive, the proposition guarantees $U_{\delta,\varepsilon}(q_{t+h})<U_{\delta,\varepsilon}(q_t)$
for all sufficiently small $h>0$, without requiring an explicit evaluation of
$U_{\delta,\varepsilon}$.  $G_{\mathbf{V}}$ is computed via finite-difference divergence of
$q_t\mathbf{V}$ on the grid.  We additionally apply a single frozen Euler step
$x_i^+=x_i+h\mathbf{V}(x_i)$ with $h=0.05$ and overlay the before- and after-step
particle clouds to give an intuitive picture of the mass redistribution.
Experimental parameters are collected in Appendix~\ref{app:exp_regional_repair}.

\begin{figure}[t]
  \centering
  \includegraphics[width=\textwidth]{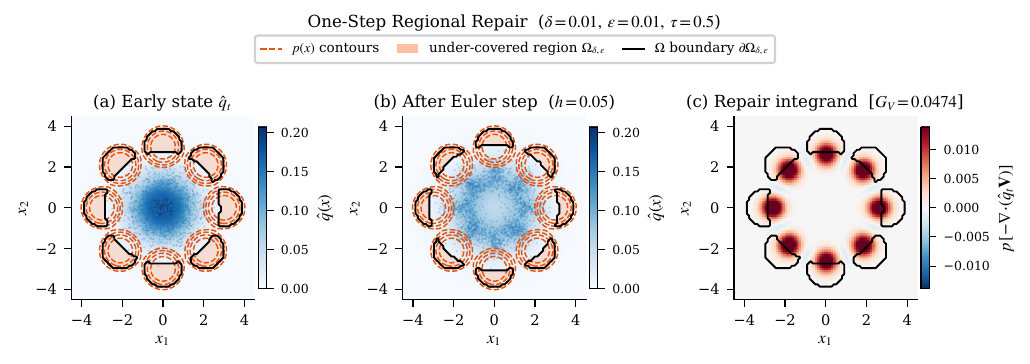}
  \caption{%
    \textbf{One-step regional-repair probe on GMM-8.}
    $\mathcal{N}(0,0.64\,I)$ particles simulate an early-training state; Reverse KL drift
    is used.  Orange shading and black boundary indicate
    $\Omega_{\delta,\varepsilon}=\{p\ge\delta,\,\hat{q}_t\le\varepsilon\}$ at the
    respective snapshot.
    (a)~Early-state KDE $\hat{q}_t$ with $\Omega$ covering all eight mode regions and $p$
    shown as dashed contours.
    (b)~KDE $\hat{q}_{t+h}$ after a single frozen Euler step ($h=0.05$): mass has moved
    outward toward the modes, visibly shrinking $\Omega$.
    (c)~Repair integrand $p[-\nabla\!\cdot\!(\hat{q}_t\mathbf{V})]$ with the early-state $\Omega$
    boundary; the annotated $G_{\mathbf{V}}=0.0474>0$ confirms
    $U_{\delta,\varepsilon}(q_{t+h})<U_{\delta,\varepsilon}(q_t)$ for any $h>0$ sufficiently
    small.}
  \label{fig:exp_regional_repair}
\end{figure}

\subsection{Ablation Studies}
\label{sec:exp_ablations}

We perform axis-level ablations on the two key hyperparameters of the main experimental
setting---the final attraction weight $a_{\mathrm{attr}}^T$ and the KDE bandwidth $\tau$---holding
all other settings fixed at the GMM-40 configuration.

\paragraph{Attraction weight and bandwidth sensitivity.}
Figure~\ref{fig:exp_ablation} summarizes the effect of independently varying $a_{\mathrm{attr}}^T$
and $\tau$ on GMM-40 (Reverse KL, Laplacian KDE), with all other settings held at the
main-experiment baseline ($a_{\mathrm{attr}}^T = 0.25$, $\tau = 1.0$).

Sweeping $a_{\mathrm{attr}}^T \in \{0.10, 0.15, 0.25, 0.40, 0.60\}$ while fixing $\tau = 1.0$
shows that all values maintain 40/40 mode coverage, but $W_1$ is minimized at the baseline
($a_{\mathrm{attr}}^T = 0.25$, $W_1 = 3.18$). Smaller values ($a_{\mathrm{attr}}^T = 0.10$)
yield lower MMD ($0.052$) yet higher $W_1$ ($4.39$), while larger values
($a_{\mathrm{attr}}^T \ge 0.40$) degrade both metrics (MMD $\ge 0.10$, $W_1 \ge 5.12$).
This is consistent with the drift structure $\mathbf{V}=w(r)\beta$: an overly strong
attraction amplifies the pull toward already-occupied modes, slowing the inter-mode mass
redistribution captured by $W_1$. Across the full range, performance degrades gradually and
no hard failure occurs.

Sweeping $\tau \in \{0.3, 0.5, 1.0, 2.0, 4.0\}$ while fixing $a_{\mathrm{attr}}^T = 0.25$
reveals a sharper asymmetry. Smaller bandwidths ($\tau = 0.3$) sharpen the KDE score and
achieve the lowest MMD ($0.023$), but sacrifice $W_1$ ($4.91$). Larger bandwidths progressively
degrade both metrics, and $\tau = 4.0$ causes complete degeneration: the over-smoothed
density loses mode-level information entirely, resulting in only 3/40 modes covered
(coverage 7.5\%). The bandwidth $\tau$ therefore exhibits a hard failure boundary on the
large-$\tau$ side, absent for $a_{\mathrm{attr}}^T$, making it the more safety-critical
hyperparameter to tune.

\paragraph{KDE variant.}
Table~\ref{tab:exp_ablation_kde} compares Gaussian and Laplacian KDE across three benchmarks
under both objectives. Under Reverse KL, Gaussian KDE achieves lower $W_1$ on GMM-8
($0.230$ vs.\ $0.259$) and lower $W_2$ on GMM-Many ($1.563$ vs.\ $1.901$), while on GMM-40
Laplacian KDE attains a marginally better $W_1$ ($3.306$ vs.\ $3.374$) at the cost of higher
MMD. Under LV, Gaussian KDE again dominates on transport cost ($W_1$/$W_2$) across all three
benchmarks, but Laplacian KDE achieves substantially better MMD on GMM-Many ($0.066$ vs.\
$0.103$), reflecting that the heavier tails of the Laplace kernel reduce over-smoothing in
the inter-mode regions. On the harder benchmarks GMM-2hard-16 and GMM-2hard-32 (not shown in
the table), Gaussian KDE diverges under both objectives, while Laplacian KDE with annealed
bandwidth successfully bridges the inter-mode gap. This suggests a practical guideline: Gaussian
KDE is often preferable on the easier low-dimensional targets we study, while Laplacian KDE is
more robust when inter-mode distances are large.

\begin{table}[t]
  \centering
  \caption{KDE variant ablation: Gaussian vs.\ Laplacian KDE on GMM-8, GMM-40, and
    GMM-Many ($d=8$, $L=8$) under both objectives. All entries report
    mean{\scriptsize\,$\pm$\,std} over 3 random seeds. Bold marks the better KDE within each
    Benchmark$\times$Objective pair. $W_1$ (exact LP) for 2D benchmarks; $W_2$ (Sinkhorn) for
    GMM-Many.}
  \label{tab:exp_ablation_kde}
  \smallskip
  \begin{tabular}{lllcc}
    \toprule
    Benchmark & Obj. & KDE & MMD$\downarrow$ & $W_1$/$W_2\!\downarrow$ \\
    \midrule
    \multirow{4}{*}{GMM-8}
      & \multirow{2}{*}{RKL}
        & Gauss   & $\mathbf{0.000}{\scriptstyle\,\pm0.000}$ & $\mathbf{0.230}{\scriptstyle\,\pm0.002}$ \\
      & & Laplace & $\mathbf{0.000}{\scriptstyle\,\pm0.000}$ & $0.259{\scriptstyle\,\pm0.009}$ \\
      & \multirow{2}{*}{LV}
        & Gauss   & $0.003{\scriptstyle\,\pm0.004}$          & $\mathbf{0.235}{\scriptstyle\,\pm0.029}$ \\
      & & Laplace & $\mathbf{0.000}{\scriptstyle\,\pm0.000}$ & $0.270{\scriptstyle\,\pm0.014}$ \\
    \midrule
    \multirow{4}{*}{GMM-40}
      & \multirow{2}{*}{RKL}
        & Gauss   & $\mathbf{0.034}{\scriptstyle\,\pm0.014}$ & $3.374{\scriptstyle\,\pm0.142}$ \\
      & & Laplace & $0.049{\scriptstyle\,\pm0.013}$          & $\mathbf{3.306}{\scriptstyle\,\pm0.293}$ \\
      & \multirow{2}{*}{LV}
        & Gauss   & $\mathbf{0.053}{\scriptstyle\,\pm0.030}$ & $\mathbf{4.219}{\scriptstyle\,\pm0.461}$ \\
      & & Laplace & $0.075{\scriptstyle\,\pm0.035}$          & $4.543{\scriptstyle\,\pm1.560}$ \\
    \midrule
    \multirow{4}{*}{\shortstack[l]{GMM-Many\\($d=8$)}}
      & \multirow{2}{*}{RKL}
        & Gauss   & $\mathbf{0.135}{\scriptstyle\,\pm0.102}$ & $\mathbf{1.563}{\scriptstyle\,\pm0.097}$ \\
      & & Laplace & $0.174{\scriptstyle\,\pm0.035}$          & $1.901{\scriptstyle\,\pm0.020}$ \\
      & \multirow{2}{*}{LV}
        & Gauss   & $0.103{\scriptstyle\,\pm0.025}$          & $\mathbf{1.557}{\scriptstyle\,\pm0.031}$ \\
      & & Laplace & $\mathbf{0.066}{\scriptstyle\,\pm0.032}$ & $2.016{\scriptstyle\,\pm0.022}$ \\
    \bottomrule
  \end{tabular}
\end{table}

\begin{figure}[t]
  \centering
  \includegraphics[width=\textwidth]{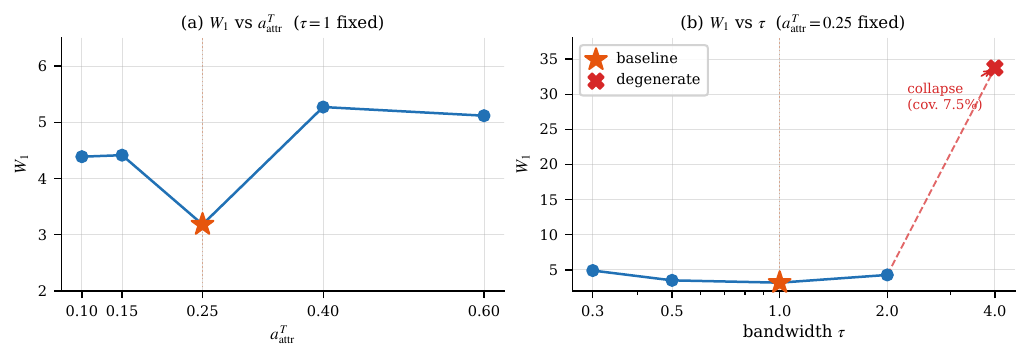}
  \caption{Sensitivity of GMM-40 performance (Reverse KL, Laplacian KDE) to the final
    attraction weight $a_{\mathrm{attr}}^T$ (a) and KDE bandwidth $\tau$ (b).
    Each axis is swept independently while the other is held at the baseline value
    ($a_{\mathrm{attr}}^T = 0.25$, $\tau = 1.0$, marked~$\star$).
    The degenerate run ($\tau = 4.0$, marked~$\times$) collapses to 7.5\% mode coverage.
    Both hyperparameters exhibit a clear working range; the baseline lies robustly within it.}
  \label{fig:exp_ablation}
\end{figure}

% ============================================================
\section{Related Work}
\label{sec:related}
% ============================================================

\textbf{Neural samplers.}
Neural samplers amortize the cost of iterative inference by training a model once against the
energy function, enabling approximate samples via a single forward pass at test time.
Normalizing flow samplers~\cite{rezende2015variational,noe2019boltzmann,wu2020stochastic} learn
invertible maps from a simple prior to the target by minimizing $\KL(q \| p)$. In the classical
static variational picture over restricted approximation families, Reverse KL is often associated
with mode-seeking behavior on multimodal targets; in practice, normalizing-flow performance is
therefore sensitive both to the objective and to the expressiveness of the invertible
architecture, which can itself limit scalability in high dimensions.
Flow Annealed Importance Sampling (FAB)~\cite{midgley2022flow} combines normalizing flows with
annealed importance sampling to partially recover coverage, and adaptive MCMC augmented with
normalizing flows~\cite{gabrie2022adaptive} uses learned flows as MCMC proposals.

Diffusion-based samplers instead learn stochastic transport trajectories from a prior to the
target via path-space objectives.
\citet{tzen2019theoretical} provide theoretical foundations by framing the
sampling problem as stochastic optimal control over latent diffusion paths.
\citet{berner2022optimal} establish the equivalence between the evidence lower
bound and the Hamilton--Jacobi--Bellman verification theorem, formally connecting diffusion
generative models to optimal control.
\citet{nusken2021solving} develop iterative diffusion optimization for
high-dimensional HJB equations, introducing a path-measure divergence framework that encompasses
a broad family of training objectives.
Building on these foundations, the Path Integral Sampler~\cite{zhangpath} casts sampling as a
stochastic optimal control problem over path space; Denoising Diffusion
Samplers~\cite{vargasdenoising} apply score-based diffusion to sampling via a learned reverse
process; Controlled Monte Carlo Diffusions~\cite{vargastransport} bridge transport and
variational inference; NETS~\cite{albergo2024nets} augments annealed importance sampling with
a learned drift trained via Fokker--Planck objectives, avoiding backpropagation through SDE
trajectories.
Training objectives have been further refined by leveraging the Log-Variance
loss~\cite{richter2024improved}, improved off-policy training~\cite{sendera2024improved}, and
iterative denoising energy matching~\cite{akhound2024iterated}.
From a geometric perspective, \citet{chemseddine2024neural} study
Fisher--Rao interpolation curves in Wasserstein geometry and identify a teleportation-of-mass
problem with linear interpolation, proposing a Langevin-flow alternative.
Despite their strong performance on multimodal benchmarks, all of these methods require
hundreds of sequential network evaluations per sample at inference time.

\textbf{Fast sampling and one-step generation.}
Reducing inference to a single function evaluation is a shared objective across the sampling
and generative modeling communities.
In the Boltzmann sampling setting, \citet{zhang2024efficient} distill a
pre-trained multi-step diffusion sampler into a one-step model via consistency
training~\cite{song2023consistency}; \citet{jutras2025one} remove the
external teacher via self-distillation but still simulate a reverse diffusion process at
training time; \citet{he2024training} train a one-step neural sampler by minimizing
the \emph{reverse diffusive KL divergence}---convolving both model and target with a
Gaussian diffusion kernel to improve mode connectivity in the noisy space---and estimate
intractable score terms via denoising score matching.
All three approaches require significant multi-step computation at training time and offer no
unified theoretical framework for coverage control.
In the data-rich generative modeling setting, consistency
models~\cite{song2023consistency} and progressive distillation~\cite{salimans2022progressive}
compress a multi-step diffusion model into a single-step generator via teacher supervision,
while MeanFlow~\cite{geng2025mean} learns mean velocity fields to achieve one-step generation
without a teacher.
The Drifting Model~\citep{deng2026drifting} takes a complementary approach: it exploits the
training-time evolution of pushforward distributions, guided by a drifting field that
continuously drives $q$ toward $p$, decoupling one-step inference from inference-time
trajectory simulation entirely.
Sinkhorn-Drifting~\citep{he2026sinkhorndriftinggenerativemodels} improves upon this by
replacing the soft coupling with a doubly-stochastic Sinkhorn transport plan, restoring
identifiability ($\mathbf{V} = 0 \Leftrightarrow p = q$) and preventing repulsion collapse
at small bandwidths.
Both remain in the data-rich setting, requiring samples from the target distribution.
Our work extends the drifting paradigm to the data-free setting and provides a unified
theoretical framework that places standard f-divergence choices under a common
Wasserstein-gradient-flow derivation.

\textbf{Stein Variational Gradient Descent.}
SVGD~\cite{liu2016stein} moves a set of particles at \emph{inference time} along the
RKHS-optimal perturbation direction
$\phi^*(x) = \mathbb{E}_{y \sim q}[k(y,x)\,\nabla_y \log p(y) + \nabla_y k(y,x)]$,
which minimizes $\KL(q \| p)$ within the RKHS unit ball.
In Section~\ref{sec:svgd_connection} we establish, via a simple integration-by-parts
identity (Lemma~\ref{lem:kernel_gradient}), that this direction equals the kernel-smoothed
version of our drifting field: $\phi^*(x) = \mathbb{E}_{y \sim q}[k(y,x)\,\beta(y)]$.
Our method recovers the unsmoothed limit $k \to \delta_x$, trading RKHS regularity for
exact pointwise evaluation of $\beta$ via autograd.
Beyond this structural difference, the two approaches operate on opposite sides of the
training/inference boundary: SVGD is a particle-based inference-time algorithm requiring
$T$ gradient evaluations per sample, whereas our framework amortizes the optimization into
a single network forward pass.

% ============================================================
\section{Conclusion}
\label{sec:conclusion}
% ============================================================

We have presented a unified framework for data-free one-step sampling grounded in the
Wasserstein gradient flow perspective. Our central contribution is the derivation that any
f-divergence functional $D_f(p \| q)$ induces a velocity field of the universal form
$\mathbf{V}(x) = w(r) \cdot \beta(x)$, where the weight function $w(r) = r^2 f''(r)$
separates a universal correction direction
$\beta(x)=\nabla\log\frac{p(x)}{q(x)}$ from an objective-dependent scalar reweighting. This
places the original Drifting Model of \citet{deng2026drifting}, the standard f-divergence
family, and the LV constructions in a common particle-transport picture.

Building on this formula, we developed a regional-response analysis centered on the soft
under-coverage mass and the associated one-step repair score. In this formulation, different
objectives are distinguished not by a change of correction direction, but by how strongly they
reweight that common direction across space. This yields a local and quantitative language for
discussing mode repair in one-step samplers and clarifies how the geometry of the current
particle cloud interacts with the chosen divergence.

On the implementation side, we showed how the same sampler--estimator interface supports the
current KDE-based data-free training pipeline and also accommodates a complementary NF-based
route as a natural extension of the framework. Within the present paper, the empirical study
focuses on the KDE instantiation. Objective-specific drifts are realized using estimated
$q$-side quantities together with batchwise normalized reweighting when the objective depends on
$r(x)$. The experiments are organized accordingly: benchmark
comparisons assess final sampling quality on multimodal Gaussian-mixture targets, illustrative
diagnostics visualize the unified drift and one-step regional repair, and ablations isolate the
roles of the KDE variant, the attraction weight, and the bandwidth.

\section{Limitations and Future Works}
Several limitations merit acknowledgment. The normalizing-flow implementation is limited by the
expressiveness of invertible architectures in high dimensions, and in the current version of the
paper this NF route is described but not yet developed into a full empirical pipeline. The
$O(N^2)$ cost of the KDE implementation is managed in practice by drawing an independent
fixed-size reference batch (controlled by $N_{\mathrm{ref}}$) for density evaluation, decoupling
evaluation cost from the training batch size. More broadly, density estimation remains the main
practical bottleneck of the data-free setting, since the realized drift quality depends directly
on the accuracy of the recovered $q$-side quantities. On the theory side, the regional-repair
and fixed-point analyses are local and structural: their sufficient conditions remain difficult
to verify in practice, and they do not yet yield full convergence or stability guarantees for the
amortized dynamics. Future directions include completing and evaluating the NF-based
implementation route, adaptive objective selection based on the current geometry of the sampler,
improved density estimators for high-dimensional energy landscapes, extensions to conditional and
structured sampling, and sharper convergence and stability guarantees for one-step training.

% ============================================================
\bibliographystyle{abbrvnat}
\bibliography{references}
% ============================================================
\newpage

% ============================================================
\appendix
\section{Additional Experimental Details and Results}
\label{app:experiments_additional}
% ============================================================

\subsection{Benchmark Definitions and Reference Distributions}
\label{app:exp_benchmarks}

All Gaussian-mixture benchmarks share the same energy form:
\[
  E(x) = -\log \sum_{k=1}^{K} w_k\,\mathcal{N}\!\bigl(x;\,\mu_k,\,\mathrm{diag}(\sigma_k^2)\bigr),
\]
where $\mathcal{N}(\cdot;\mu,\Sigma)$ denotes the Gaussian density.
Reference samples are drawn exactly from the analytic mixture ($10{,}000$ samples unless otherwise noted).

\paragraph{GMM-8.}
A two-dimensional target with $K=8$ modes arranged in a cross-and-diagonal layout:
\begin{align*}
  &\mu_{\{1,2\}} = (\pm3,\,0),\quad \mu_{\{3,4\}} = (0,\,\pm3),\\
  &\mu_{\{5,6,7,8\}} = (\pm2.1,\,\pm2.1)\quad\text{(all four sign combinations)}.
\end{align*}
All modes share the same isotropic standard deviation $\sigma=0.4$ and equal weight $w_k=1/8$.

\paragraph{GMM-40.}
A two-dimensional target with $K=40$ modes. Mode centers $\mu_k$ are drawn uniformly from $[-40,\,40]^2$ using a fixed random seed ($\texttt{seed}=42$), giving a deterministic but pseudo-random layout. All modes share isotropic standard deviation $\sigma=\mathrm{softplus}(1)\approx1.313$ and equal weight $w_k=1/40$.

\paragraph{GMM-Manymodes (8D).}
An eight-dimensional target with $K=8$ modes. Centers $\mu_k$ are drawn uniformly from $[-8,\,8]^8$ with $\texttt{seed}=42$. All modes share isotropic standard deviation $\sigma=\sqrt{0.5}\approx0.707$. Mixture weights are geometrically increasing:
\[
  w_k \propto 3^{(k-1)/7}, \quad k=1,\ldots,8,
\]
so the weight ratio between the most and least probable mode is $3{:}1$. This benchmark tests coverage of an imbalanced mode structure in higher dimension.

\paragraph{GMM-2hard (16D and 32D).}
Two bimodal targets in $\mathbb{R}^{16}$ and $\mathbb{R}^{32}$. Both have $K=2$ modes with means
\[
  \mu_1 = +\mathbf{1}_d,\qquad \mu_2 = -\mathbf{1}_d,
\]
where $\mathbf{1}_d=(1,\ldots,1)^\top\in\mathbb{R}^d$, and mixture weights $w_1=2/3$, $w_2=1/3$.
The covariance is diagonal and anisotropic, with per-dimension standard deviations
\[
  \sigma_j = \sqrt{0.05\cdot 10^{-2(d-j)/(d-1)}}, \quad j=1,\ldots,d,
\]
yielding a condition number of approximately $100$ (smallest $\sigma_j\approx0.022$, largest $\sigma_j\approx0.224$).
The mode separation $\|\mu_1-\mu_2\|=2\sqrt{d}$ grows with dimension ($8.0$ at $d=16$, $11.3$ at $d=32$), making mode-bridging the primary challenge.

\subsection{Architectures and Training Details}
\label{app:exp_training}

\paragraph{Generator architecture.}
All experiments use a residual MLP sampler (\texttt{ResMLPSampler}) with sinusoidal input
embeddings of size $128$, skip connections between every hidden layer, and a linear output layer.
All benchmarks use $5$ hidden layers of width $128$.
The latent space dimension equals the target ambient dimension for all GMM benchmarks, except
GMM-2hard-16 where a half-dimension latent ($d_z=8$ for target $d=16$) is used.
No batch normalization or layer normalization is applied; the sinusoidal embedding is sufficient
to break permutation symmetry and stabilize training.

\paragraph{Optimizer.}
We use the Adam optimizer~\citep{kingma2014adam} with $\beta_1=0.9$, $\beta_2=0.999$, $\epsilon=10^{-8}$.
The learning rate is fixed at $2\times10^{-3}$ throughout training (no scheduler) and no gradient
clipping is applied.

\paragraph{Bandwidth schedule.}
For GMM-8, the bandwidth is annealed with a cosine schedule from $\tau_{\mathrm{init}}=0.5$
(covering the inter-mode scale) to $\tau_{\mathrm{final}}=0.15$ (within-mode scale).
For GMM-40 and GMM-Manymodes a constant bandwidth calibrated to the mode spacing is used
($\tau=12$ and $\tau=25$ respectively). For GMM-2hard the bandwidth is annealed from an initial value $\tau_{\mathrm{init}}$ chosen
to cover the full mode-separation distance down to a smaller final value, using a full-range cosine
schedule; as the separation scales as $2\sqrt{d}$, the optimal $\tau_{\mathrm{init}}$ differs substantially
between the 16D and 32D variants (see Table~\ref{tab:hyperparams}).

\paragraph{Attraction-weight schedule.}
The attraction weight $a_{\mathrm{attr}}$ is annealed from a small initial value $a_0$ to a
larger final value $a_T$ using a linear or cosine schedule over a central window
$[t_{\mathrm{start}}, t_{\mathrm{stop}}] \subset [0,T]$. The initial value $a_0$ is deliberately
small (0.01--0.1) so that the repulsion term dominates during early training and particles spread
across the target support before the target-side pull begins to sharpen modes. The final value $a_T$ is set manually for all benchmarks; the GMM-2hard variants use a large
final value ($a_T=2.0$ and $4.0$ for 16D and 32D respectively) to drive the generator across
the large inter-mode gap.

\paragraph{Per-benchmark hyperparameters.}
Table~\ref{tab:hyperparams} lists the key hyperparameters for the best-performing configuration
on each benchmark reported in the main text.

\begin{table}[h]
  \centering
  \caption{Hyperparameters for each benchmark. $H$: hidden width; $L$: number of hidden layers;
    $B$: batch size; $T$: training steps; $\tau$: KDE bandwidth (init $\to$ final);
    $a_{\mathrm{attr}}$: attraction weight (init $\to$ final). ``const'' indicates a constant schedule;
    ``$\to$'' indicates a cosine or linear anneal over a central window.}
  \label{tab:hyperparams}
  \smallskip
  \resizebox{\textwidth}{!}{%
  \begin{tabular}{lllcccll}
    \toprule
    Benchmark & Obj. & Estimator & $H{\times}L$ & $B$ & $T$ & $\tau$ & $a_{\mathrm{attr}}$ \\
    \midrule
    GMM-8       & RKL & Laplace KDE & $128{\times}5$ & 4096 & 8000  & $0.5{\to}0.15$ (cosine) & $0.1{\to}0.1$ (cosine) \\
    GMM-40      & RKL & Laplace KDE & $128{\times}5$ & 4096 & 8000 & 1.0 (const)            & $0.1{\to}0.25$ (linear) \\
    GMM-40      & LV  & Laplace KDE & $128{\times}5$ & 4096 & 8000 & 1.0 (const)            & $0.1{\to}0.25$ (linear) \\
    GMM-Manymodes & RKL & Laplace KDE & $128{\times}5$ & 1024 & 2000 & 25.0 (const)        & $0.02{\to}0.04$ (lin) \\
    GMM-2hard-16 & RKL & Laplace KDE & $128{\times}5$ & 4096 & 2000 & $0.004{\to}0.006$ (cos) & $0.01{\to}2.0$ (cosine) \\
    GMM-2hard-32 & RKL & Laplace KDE & $128{\times}5$ & 8192 & 2000 & $0.001{\to}0.0015$ (cos) & $0.05{\to}4.0$ (cosine)  \\
    \bottomrule
  \end{tabular}}
\end{table}

All experiments use Adam with $\mathrm{lr}=2{\times}10^{-3}$,
precision float32, and a single GPU. Results are reported over 3 random seeds per benchmark.

\subsection{Evaluation Metrics}
\label{app:exp_metrics}

All metrics are computed between $N=2{,}000$ model samples and $N=2{,}000$ reference samples
drawn from the analytic target unless otherwise noted.

\paragraph{Maximum Mean Discrepancy (MMD).}
We use the unbiased two-sample estimator with a radial basis function kernel
$k(x,y)=\exp(-\|x-y\|^2/2\sigma^2)$ where $\sigma$ is set by the median heuristic: $\sigma^2=\mathrm{median}(\{\|x_i-y_j\|^2\})$ over cross-product pairs. The unbiased estimator is
\[
  \widehat{\mathrm{MMD}}^2(P,Q)
  = \frac{1}{N(N-1)}\sum_{i\neq j}k(x_i,x_j)
    + \frac{1}{N(N-1)}\sum_{i\neq j}k(y_i,y_j)
    - \frac{2}{N^2}\sum_{i,j}k(x_i,y_j),
\]
where $x_i\sim P$ and $y_j\sim Q$.

\paragraph{1-Wasserstein Distance ($W_1$).}
For two-dimensional targets (GMM-8 and GMM-40) we solve the exact discrete linear program using the Python Optimal Transport (POT) library~\citep{flamary2021pot}. For higher-dimensional targets, $W_1$ is not reported.

\paragraph{2-Wasserstein Distance ($W_2$).}
We compute the Sinkhorn approximation to $W_2^2$ with entropic regularization $\varepsilon=0.05$, using at most $200$ Sinkhorn iterations. We report $W_2=\sqrt{W_2^2}$. For two-dimensional targets, POT's exact solver is used when sample size permits; for higher dimensions, only the Sinkhorn estimate is reported.

\paragraph{Mode Coverage.}
Mode $k$ is declared covered if at least $1\%$ of the $2{,}000$ model samples fall within an $\ell_2$ ball of radius $r=3\sigma$ centered at $\mu_k$. Coverage is reported as the fraction of covered modes over all $K$ modes (e.g., $8/8$ or $40/40$).

\subsection{Protocol for the Unified Drift and Compression-Elasticity Figures}
\label{app:exp_unified_drift}

The unified drift figure (Figure~\ref{fig:exp_unified_drift}) is produced from a single
converged checkpoint trained on GMM-8 with Reverse KL and Laplacian KDE.

\paragraph{Checkpoint and grid.}
The generator is evaluated at the end of a dedicated visualization run (step 20\,000) with
$N=2000$ latent samples ($z\sim\mathcal{N}(0,I)$). This diagnostic run is separate from the main
benchmark tables and is used only to visualize the converged drift geometry on GMM-8. All fields
are evaluated on an $80\times 80$ grid covering
$[-4.5,4.5]^2$.  The Laplacian KDE bandwidth is held at the converged value
$\tau_{\mathrm{conv}}=\tau_{\mathrm{init}}\times\mathrm{final\_ratio}=0.5\times0.3=0.15$.

\paragraph{Field computation.}
The score $\nabla\!\log p$ is computed analytically from the GMM energy.  The KDE
log-density $\log\hat{q}_t=\log\!\sum_i k(x,x_i)-\log N$ is normalised by $N$ before
taking the score $\nabla\!\log\hat{q}_t$.  The correction direction is
$\beta=\nabla\!\log p - \nabla\!\log\hat{q}_t$.  The density ratio on the grid is
$r(x)=\exp(\log p(x)-\log\hat{q}_t(x))$.  The Forward KL weight is
$w_{\mathrm{FKL}}(r)=r$ (grid-normalised over the evaluation grid as $\bar{w}_i=w_i/\mathrm{mean}(w)$ to put it on
a comparable scale to $w\equiv 1$); the LV weight is
$w_{\mathrm{LV}}=2(1+[m-\bar m]_+)$ where $m=\log r$ and $\bar m$ is the mean over all $80\times80$ grid points.
The three drift fields are $\mathbf{V}=w\cdot\beta$ for $w\in\{1,\bar{w}_{\mathrm{FKL}},w_{\mathrm{LV}}\}$.

\paragraph{Compression-elasticity computation.}
The geometric compression index is
$\kappa_t(x)=-\nabla\!\cdot\!(\hat{q}_t\beta)(x)\;/\;\hat{q}_t(x)\,\|\beta(x)\|^2$,
evaluated via central-difference divergence (one-sided at boundary cells).
The divergence elasticities are $\eta_f\in\{0,1,2\}$ for Reverse KL, Forward KL, and
$\chi^2$, respectively (all constants; see Table~\ref{tab:regional_repair_examples}).
The gap $\kappa_t-\eta_f$ is plotted on a per-divergence 96th-percentile symmetric colour
scale with $\mathrm{RdBu\_r}$; the zero-crossing contour is overlaid in black.
The scalar $G_f$ is the Riemann sum
$G_f=\sum_{x\in\text{grid}}p(x)\,\hat{q}_t(x)\,w(r(x))\,\|\beta(x)\|^2\,[\kappa_t(x)-\eta_f]\cdot\Delta x\,\Delta y$.

\subsection{Protocol for the Frozen One-Step Regional-Repair Probe}
\label{app:exp_regional_repair}

The frozen one-step probe in Section~\ref{sec:exp_demos} uses the following concrete choices.

\paragraph{Snapshot selection.}
Rather than extracting an intermediate training checkpoint, we use $N=2000$ particles drawn
i.i.d.\ from $\mathcal{N}(0,0.64\,I_2)$ as a proxy for an early-training state.
This distribution is concentrated near the origin and leaves all eight modes of GMM-8 strongly
under-covered, making the repair signal clearly visible.

\paragraph{Drift estimation.}
The drift $\mathbf{V}(x)=\beta(x)=\nabla\!\log p(x)-\nabla\!\log q_t(x)$ is evaluated on an
$80\times 80$ grid covering $[-4.5,4.5]^2$.  The score $\nabla\!\log p$ is computed analytically
from the GMM energy.  The KDE score $\nabla\!\log\hat{q}_t$ is obtained from the Laplacian KDE
used in training with bandwidth fixed at $\tau=\tau_{\mathrm{init}}=0.5$, reflecting the
step-0 KDE state that is consistent with the early-training proxy particles.
No generator forward pass is required.

\paragraph{Under-coverage thresholds.}
The region $\Omega_{\delta,\varepsilon}$ is identified with $\delta=0.01$ and
$\varepsilon=0.01$.  The KDE density is normalised as
$\hat{q}_t(x)=(1/N)\sum_{i=1}^N k(x,x_i)$ so that $\hat{q}_t$ integrates to approximately
unity; $\varepsilon$ is set relative to this normalised scale.  With $N(0,0.64I)$ proxy
particles and bandwidth $\tau=0.5$ (Laplacian kernel), the normalised KDE density at
the GMM-8 mode centres ($r\approx3$) is $\hat{q}_t\approx0.002<\varepsilon$, so all eight
mode bodies fall inside $\Omega$; the origin region has $\hat{q}_t\approx0.03>\varepsilon$
and is correctly excluded as already well-covered.

\paragraph{Repair score computation.}
$G_{\mathbf{V}}(t;\delta,\varepsilon)=\int_{\Omega}p(x)[-\nabla\!\cdot\!(\hat{q}_t\mathbf{V})(x)]\,dx$
is approximated on the grid via central-difference divergence of the vector field
$\hat{q}_t(x)\mathbf{V}(x)$ (one-sided differences at boundary cells), followed by a masked
Riemann sum $\sum_{x\in\Omega}\mathrm{integrand}(x)\cdot\Delta x\,\Delta y$ over grid cells
in $\Omega_{\delta,\varepsilon}$.

\paragraph{Euler step.}
A single frozen Euler step $x_i^+=x_i+h\,\mathbf{V}(x_i)$ with $h=0.05$ is applied at
each particle position; network weights are not updated.  The KDE density
$\hat{q}_{t+h}=(1/N)\sum_i k(\cdot,x_i^+)$ of the transported particles is re-evaluated on
the same grid.  A separate after-step mask
$\Omega_{\delta,\varepsilon}(q_{t+h})=\{p\ge\delta,\,\hat{q}_{t+h}\le\varepsilon\}$
is computed from $\hat{q}_{t+h}$; comparing this to the before-step mask visually confirms
that $\Omega$ shrinks after one step (Figure~\ref{fig:exp_regional_repair}b).
Panel~(c) uses the before-step mask throughout, since $G_{\mathbf{V}}$ is defined as an
integral over $\Omega_{\delta,\varepsilon}(q_t)$.

\section{Proof of Theorem 1 and the Reverse-KL/Fokker--Planck Special Case}
\label{app:proof_unified}
% ============================================================

\subsection{General f-Divergence Derivation}
We provide the full proof of Theorem~\ref{thm:unified}. Let
\[
  \cF(q)=D_f(p\|q)=\int q(x)\,f\!\left(\frac{p(x)}{q(x)}\right)\,dx,
  \qquad r(x)=\frac{p(x)}{q(x)}.
\]
For an admissible perturbation $q\mapsto q+\varepsilon\phi$,
\begin{align}
  \frac{d}{d\varepsilon}\bigg|_{\varepsilon=0} D_f(p\|q+\varepsilon\phi)
  &= \int \phi(x)\,\frac{\delta D_f}{\delta q(x)}\,dx.
\end{align}
Differentiating the integrand gives
\begin{align}
  \frac{d}{d\varepsilon}\bigg|_{\varepsilon=0}
  (q+\varepsilon\phi)\,f\!\left(\frac{p}{q+\varepsilon\phi}\right)
  &= \phi f(r) + q f'(r)\left(-\frac{p\phi}{q^2}\right) \notag\\
  &= \phi\bigl[f(r)-r f'(r)\bigr],
\end{align}
and therefore
\begin{equation}
  \frac{\delta D_f}{\delta q}(x)=f(r(x)) - r(x) f'(r(x)).
\end{equation}
By Otto's calculus,
\[
  \mathbf{V}(x)=-\nabla_x\frac{\delta D_f}{\delta q(x)}.
\]
If $g(r)=f(r)-r f'(r)$, then $g'(r)=-r f''(r)$, hence
\begin{align}
  \mathbf{V}(x)
  &= -\nabla_x g(r(x))
   = -g'(r(x))\,\nabla_x r(x)
   = r(x) f''(r(x))\,\nabla_x r(x).
\end{align}
Using $\nabla_x r(x)=r(x)\nabla\log r(x)=r(x)\beta(x)$, we obtain
\begin{align}
  \mathbf{V}(x)
  = r(x)^2 f''(r(x))\,\beta(x)
  = w(r(x))\,\beta(x),
\end{align}
which proves Theorem~\ref{thm:unified}. $\square$

\subsection{Special Case: Reverse KL and the Fokker--Planck Equation}

For Reverse KL, the $f$-divergence generator is $f(r)=-\log r$, so
\[
  f''(r)=\frac{1}{r^2},
  \qquad
  w(r)=r^2 f''(r)=1.
\]
Theorem~\ref{thm:unified} therefore yields the Wasserstein gradient-flow velocity field
\begin{equation}
  V_t(x)=\beta(x)=\nabla\log p(x)-\nabla\log q_t(x)
  = -\nabla E(x)-\nabla\log q_t(x).
  \label{eq:kl_velocity}
\end{equation}
Substituting this into the continuity equation
\[
  \partial_t q_t + \nabla\cdot(q_tV_t)=0
\]
gives
\begin{align}
  q_tV_t
  &= q_t\bigl(-\nabla E-\nabla\log q_t\bigr)
   = -q_t\nabla E - \nabla q_t,
\end{align}
and hence
\begin{equation}
  \partial_t q_t
  = \nabla\cdot(q_t\nabla E) + \Delta q_t.
  \label{eq:fokker_planck_derived}
\end{equation}
This is exactly the Kolmogorov forward (Fokker--Planck) equation associated with the
overdamped Langevin SDE
\[
  dX_t = -\nabla E(X_t)\,dt + \sqrt{2}\,dW_t.
\]

\begin{remark}[Stationary distribution]
  Setting $\partial_t q_t = 0$ in~\eqref{eq:fokker_planck_derived} requires
  $\nabla\cdot(\nabla q + q\nabla E)=0$. Under the standard zero-flux boundary condition,
  this is satisfied by $q \propto e^{-E} = p$, confirming that $p$ is a stationary
  distribution of the flow.
\end{remark}

% ============================================================
\section{LV Divergence: Full Derivations}
\label{app:lv}
% ============================================================

\subsection{Case 1: $\nu = q$}

Let $\ell(x) \coloneqq \log\tilde{r}(x)$ and $\mu_q \coloneqq \E_q[\ell]$ for brevity.
The functional $\cL(q) = \E_q[\ell^2] - \mu_q^2$ has both the integration measure $q$ and
the integrand $\ell(x) = \log\tilde{p}(x) - \log q(x)$ depending on $q$. Under perturbation
$q \to q + \epsilon \delta q$, we have $\delta \ell = -\delta q / q$. Computing:
\begin{align}
  \delta \E_q[\ell^2] &= \int (\ell^2 - 2\ell) \delta q\, dx, \\
  \delta (\mu_q^2) &= 2\mu_q \int (\ell - 1) \delta q\, dx.
\end{align}
Hence, up to an additive constant independent of $x$,
\begin{align}
  \frac{\delta \cL}{\delta q}
  &= \ell^2 - 2\ell - 2\mu_q(\ell-1) \notag \\
  &= (\ell-\mu_q)^2 - 2(\ell-\mu_q) + \text{const}.
\end{align}
Taking the negative spatial gradient and discarding additive constants:
\begin{align}
  \mathbf{V}^{\nu=q}(x) = -\nabla_x\!\bigl[(\ell - \mu_q)^2 - 2(\ell - \mu_q)\bigr]
  = 2\bigl[1 - (\log\tilde{r}(x) - \E_q[\log\tilde{r}])\bigr] \cdot \beta(x),
\end{align}
as stated in Eq.~\eqref{eq:lv_case1}. $\square$

\subsection{Case 2: $\nu = p$}

With $\nu = p$ independent of $q$, only $\delta \ell = -\delta q/q$ varies:
\begin{align}
  \delta \E_p[\ell^2] &= -2\int \frac{p}{q} \ell\, \delta q\,dx, \quad
  \delta \bigl(\E_p[\ell]^2\bigr) = -2\E_p[\ell] \int \frac{p}{q} \delta q\,dx.
\end{align}
Thus $\frac{\delta \cL}{\delta q} = -\frac{2p}{q}(\ell - \E_p[\ell])$.
Taking the negative spatial gradient of $\frac{p}{q}(\ell - \E_p[\ell])$ via product rule:
\begin{align}
  \mathbf{V}^{\nu=p}(x) = 2 \frac{p(x)}{q(x)} \bigl[1 + (\log\tilde{r}(x) - \E_p[\log\tilde{r}])\bigr] \cdot \beta(x). \quad \square
\end{align}

\subsection{Case 3: $\nu$ independent of $q$}

Let $\mu_\nu \coloneqq \E_\nu[\ell]$, where now $\nu$ is fixed independently of the current
$q$. As in Case~2, only the integrand varies under perturbation, so
\begin{align}
  \delta \E_\nu[\ell^2] &= -2\int \frac{\nu}{q}\,\ell\,\delta q\,dx, \\
  \delta \bigl(\E_\nu[\ell]^2\bigr) &= -2\mu_\nu \int \frac{\nu}{q}\,\delta q\,dx.
\end{align}
Hence
\begin{align}
  \frac{\delta \cL}{\delta q}
  = -2\frac{\nu}{q}\,(\ell-\mu_\nu).
\end{align}
Taking the negative spatial gradient and applying the product rule gives
\begin{align}
  \mathbf{V}^{\nu}(x)
  &= -\nabla_x\!\left[-2\frac{\nu(x)}{q(x)}(\ell(x)-\mu_\nu)\right] \notag\\
  &= 2\nabla_x\!\left[\frac{\nu(x)}{q(x)}(\ell(x)-\mu_\nu)\right] \notag\\
  &= 2\left[
      \frac{\nu(x)}{q(x)}\,\nabla_x\ell(x)
      + (\ell(x)-\mu_\nu)\,\nabla_x\frac{\nu(x)}{q(x)}
    \right] \notag\\
  &= 2\frac{\nu(x)}{q(x)}\left[
      \beta(x)
      + (\log\tilde{r}(x)-\E_\nu[\log\tilde{r}])\,\nabla_x\log\frac{\nu(x)}{q(x)}
    \right],
\end{align}
which is Eq.~\eqref{eq:lv_case3}. In the stop-gradient special case $\nu=q_{\mathrm{sg}}$,
one has $\nu/q\equiv 1$, so the shape-correction term vanishes and the field reduces to
$2\beta(x)$. $\square$

\subsection{Z-Annihilator}
\label{app:z_annihilator}

Since $\log\tilde{r}(x) = \log r(x) + \log Z$, the centered quantity satisfies:
\begin{align}
  \log\tilde{r}(x) - \E_\nu[\log\tilde{r}]
  &= \bigl(\log r(x) + \log Z\bigr) - \bigl(\E_\nu[\log r] + \log Z\bigr) \notag \\
  &= \log r(x) - \E_\nu[\log r].
\end{align}
Hence the unknown normalizing constant $Z$ cancels exactly after centering. This removes the
partition-function ambiguity, although practical tractability still depends on whether the
remaining expectation under $\nu$ can be approximated in the data-free setting. $\square$

% ============================================================
\section{KDE Scores and Recovery of the Drifting Model}
\label{app:recovery}
% ============================================================

Setting $\cF = \KL(q\|p)$ gives $w(r) = 1$ and hence
\[
  \mathbf{V}(x)
  =
  \beta(x)
  =
  \nabla\log r(x)
  =
  \nabla\log p(x)-\nabla\log q(x)
  =
  -\nabla E(x)-\nabla\log q(x).
\]
To connect this field with the original Drifting Model, replace $q$ by the mini-batch KDE
\[
  \hat q(x) = \frac{1}{N}\sum_{j=1}^N k(x,x_j).
\]
Its score always has the form
\[
  \nabla_x \log \hat q(x)
  =
  \frac{\sum_j \nabla_x k(x,x_j)}{\sum_j k(x,x_j)}.
\]
The empirical Reverse-KL drift is therefore
\[
  \widehat{\mathbf V}(x) = -\nabla E(x)-\nabla\log\hat q(x).
\]

\paragraph{Gaussian kernel (exact score).}
For $k(x,y) = \exp(-\|x-y\|^2/\tau)$, the kernel gradient is
\[
  \nabla_x k(x,y) = k(x,y)\,\frac{2(y-x)}{\tau}.
\]
Substituting this identity gives
\begin{align}
  \nabla_x \log \hat{q}(x)
  &= \frac{\sum_j \nabla_x k(x,x_j)}{\sum_j k(x,x_j)}
   = \frac{\sum_j k(x,x_j)\,\tfrac{2(x_j - x)}{\tau}}{\sum_j k(x,x_j)}
   = \frac{2}{\tau} \sum_j W_j\,(x_j - x)
   = \frac{2}{\tau}\,m_q(x),
\end{align}
where $W_j = \mathrm{softmax}_j(-\|x-x_j\|^2/\tau)$ and $m_q(x) \coloneqq \sum_j W_j(x_j - x)$
is the \emph{mean-shift vector}. The equality is exact: no approximation is involved.
The drifting field therefore becomes
\begin{equation}
  \widehat{\mathbf V}(x) = -\nabla E(x) - \frac{2}{\tau}\,m_q(x).
  \label{eq:recovery_gaussian}
\end{equation}
Thus the Gaussian mean-shift formula used in our Reverse-KL discussion is exactly the score of the
finite-sample RBF KDE.

\paragraph{Laplace kernel.}
For $k(x,y)=\exp(-\|x-y\|/\tau)$,
\[
  \nabla_x k(x,y)
  =
  k(x,y)\,\frac{y-x}{\tau\|y-x\|},
\]
so the exact Laplace-KDE score is
\begin{align}
  \nabla_x \log \hat q(x)
  =
  \frac{1}{\tau}
  \sum_j W_j^{\mathrm{lap}}\,
  \frac{x_j-x}{\|x_j-x\|},
\end{align}
where $W_j^{\mathrm{lap}}=\text{softmax}_j(-\|x-x_j\|/\tau)$.
The original Drifting Model~\citep{deng2026drifting} uses the same Laplace weights but computes
mean-shift with \emph{unnormalized} displacements:
\begin{align}
  m_q(x_i) = \sum_j \mathrm{softmax}_j\!\left(-\frac{\|x_i - x_j\|}{\tau}\right)(x_j - x_i).
\end{align}
Because the exact Laplace score uses unit displacements whereas the Drifting Model uses full
displacements, the latter does not recover $\nabla\log\hat q$ exactly; the two vectors are in
general neither proportional nor aligned. The Drifting Model's repulsion therefore combines
Laplace-based weights with the unnormalized displacement structure familiar from RBF mean-shift
and should be interpreted as a related heuristic realization rather than an exact KDE-score
identity.

\paragraph{Conclusion.}
The Drifting Model of \citet{deng2026drifting} uses the Laplace mean-shift
$m_q$ (unnormalized displacements, Laplace weights) combined with the stop-gradient
objective~\eqref{eq:drifting_loss}.
Substituting the RBF kernel instead yields the exact score
($\nabla\log\hat{q} = \frac{2}{\tau}m_q^{\mathrm{rbf}}$); the Laplace mean-shift is an
approximation that does not correspond to any standard KDE gradient.
In both cases, the underlying functional is Reverse KL ($w = 1$), placing the Drifting
Model at the Reverse-KL endpoint, which also appears as the constant-weight
$\alpha\to 0^+$ limit of the Tsallis family. $\square$

% ============================================================
\section{Connection to Stein Variational Gradient Descent}
\label{app:svgd}
% ============================================================

Our drift field $\beta(x) = \nabla \log r(x)$ bears a precise mathematical
relationship to the optimal perturbation direction of Stein Variational Gradient Descent
(SVGD)~\cite{liu2016stein}. We establish this through a kernel gradient identity.

\begin{lemma}[Kernel Gradient Identity]
  \label{lem:kernel_gradient}
  Let $q$ be a smooth density on $\R^d$ and $k: \R^d \times \R^d \to \R$ a smooth kernel
  satisfying $k(y, x)\,q(y) \to 0$ as $\|y\| \to \infty$. Then for any fixed $x \in \R^d$:
  \begin{equation}
    \mathbb{E}_{y \sim q}\bigl[\nabla_y k(y, x)\bigr]
    = -\,\mathbb{E}_{y \sim q}\bigl[k(y, x)\,\nabla_y \log q(y)\bigr].
    \label{eq:kernel_gradient}
  \end{equation}
\end{lemma}

\begin{proof}
  Integration by parts:
  \begin{align}
    \mathbb{E}_{y \sim q}[\nabla_y k(y, x)]
    &= \int \nabla_y k(y, x)\cdot q(y)\,dy
     = -\int k(y, x)\,\nabla_y q(y)\,dy
       + \underbrace{\bigl[k(y,x)\,q(y)\bigr]_{\partial}}_{=\,0} \notag \\
    &= -\int k(y, x)\,\nabla_y \log q(y)\,q(y)\,dy
     = -\,\mathbb{E}_{y \sim q}[k(y, x)\,\nabla_y \log q(y)]. \qedhere
  \end{align}
\end{proof}

\begin{corollary}[SVGD as Kernel-Smoothed Drift Field]
  \label{cor:svgd}
  Under the conditions of Lemma~\ref{lem:kernel_gradient}, the SVGD optimal perturbation
  direction (\citet{liu2016stein}, Lemma~3.2) satisfies:
  \begin{equation}
    \phi^*_{q,p}(x)
    = \mathbb{E}_{y \sim q}\bigl[k(y, x)\,\nabla_y \log r(y)\bigr]
    = \mathbb{E}_{y \sim q}\bigl[k(y, x)\,\beta(y)\bigr].
    \label{eq:svgd_drift}
  \end{equation}
  That is, $\phi^*_{q,p}(x)$ is the kernel-weighted expectation of $\beta$ under $q$.
\end{corollary}

SVGD moves each particle along a \emph{smoothed} drift signal, averaging $\beta(y)$ over the
neighborhood of $x$. This RKHS smoothing guarantees $\phi^* \in \mathcal{H}^d$ but introduces
a bias: the particle at $x$ responds to the drift at nearby $y$, not at $x$ itself.
Our framework evaluates $\beta$ exactly at each generator output via autograd, recovering the
unsmoothed limit
\begin{equation}
  \phi^*_{q,p}(x) \;\xrightarrow{k \to \delta_x}\; \beta(x).
  \label{eq:svgd_limit}
\end{equation}
Substituting $\phi = \beta$ into Liu and Wang's KL derivative formula yields the Fisher
divergence $\mathcal{F}(q,p) = \mathbb{E}_q[\|\beta\|_2^2]$, whereas SVGD descends
$\KL(q\|p)$ in the RKHS unit ball. Table~\ref{tab:svgd_compare} summarizes the comparison.

\begin{table}[h]
  \centering
  \caption{Comparison between SVGD and our framework for the Reverse KL case ($w = 1$).}
  \label{tab:svgd_compare}
  \smallskip
  \begin{tabular}{lll}
    \toprule
    & SVGD~\cite{liu2016stein} & Ours (Reverse KL) \\
    \midrule
    Update signal & $\phi^*(x) = \mathbb{E}_{y\sim q}[k(y,x)\,\beta(y)]$ & $\mathbf{V}(x) = \beta(x)$ \\
    Score of $p$ & Kernel-smoothed & Exact via autograd \\
    Score of $q$ & Implicit via $\nabla_y k$ & KDE estimate $\nabla\!\log\hat{q}(x)$ \\
    Objective & $\KL(q\|p)$ in $\mathcal{H}^d$ unit ball & Fisher div.\ $\mathcal{F}(q,p)$ \\
    Inference & Iterative ($T$ evals) & One-step ($1$ eval) \\
    Architecture & Particle set & Generator $f_\theta$ \\
    \bottomrule
  \end{tabular}
\end{table}

% ============================================================
\section{Proofs for Mode Coverage Theory and Fixed-Point Analysis}
\label{app:mode_coverage_proofs}
% ============================================================

\subsection{Auxiliary first-order identities}

Fix a current density $q_t$ and consider the one-step Euler update
\[
  x^+ = x + h\,\mathbf{V}^{D_f}(x),
\]
with $\mathbf{V}^{D_f}(x)=w(r(x))\,\beta(x)$ and
$D_{q_t}(x)=\log(p(x)/q_t(x))=\log r(x)$. Since $D_{q_t}$ is differentiable, Taylor expansion
gives
\begin{equation}
  D_{q_t}(x^+) = D_{q_t}(x) + h\,\mathbf{V}^{D_f}(x)\cdot\nabla D_{q_t}(x) + o(h).
  \label{eq:coverage_classification}
\end{equation}
Using $\nabla D_{q_t}(x)=\beta(x)$, we obtain
\[
  \mathbf{V}^{D_f}(x)\cdot\nabla D_{q_t}(x)
  = w(r(x))\,\beta(x)\cdot\beta(x)
  = w(r(x))\,\|\beta(x)\|^2,
\]
which proves \eqref{eq:coverage_classification}.

For the density identity, the continuity equation
\[
  \partial_t q_t + \nabla\cdot(q_t\mathbf{V}^{D_f}) = 0
\]
implies the Euler expansion
\begin{equation}
  q_{t+h}(x)=q_t(x)-h\,\nabla\cdot\!\bigl(q_t(x)\,\mathbf{V}^{D_f}(x)\bigr)+o(h),
  \label{eq:density_response_1}
\end{equation}
which is \eqref{eq:density_response_1}.  Substituting $\mathbf{V}^{D_f}=w(r)\beta$ gives
\[
  \nabla\cdot(q_tw(r)\beta)
  =
  w(r)\,\nabla\cdot(q_t\beta)
  +
  q_t\,\nabla w(r)\cdot\beta.
\]
Since $w=w(r)$ and $\nabla r = r\,\beta$, one has
\[
  \nabla w(r)=w'(r)\nabla r = w'(r)\,r\,\beta,
\]
and therefore
\[
  q_t\,\nabla w(r)\cdot\beta
  =
  q_t\,w'(r)\,r\,\|\beta\|^2,
\]
Therefore,
\begin{equation}
  q_{t+h}(x)-q_t(x)
  =
  -h\,w(r)\,\nabla\cdot(q_t\beta)
  -h\,q_t\,w'(r)\,r\,\|\beta\|^2
  +o(h).
  \label{eq:density_response_2}
\end{equation}
which proves \eqref{eq:density_response_2}.

\subsection{Proof of Proposition~\ref{prop:soft_undercoverage}}

\begin{proof}[Proof of Proposition~\ref{prop:soft_undercoverage}]
  Write $\xi_t(x)\coloneqq -\nabla\cdot(q_t\mathbf{V})(x)$. For a general drift field
  $\mathbf V$, the continuity equation gives
  \[
    q_{t+h}(x)=q_t(x)+h\,\xi_t(x)+o(h).
  \]
  Let $\phi_\varepsilon(z)\coloneqq [\varepsilon-z]_+$.  For every $x$ with
  $q_t(x)\neq\varepsilon$, the function $\phi_\varepsilon$ is differentiable at $q_t(x)$ and
  satisfies
  \[
    \phi_\varepsilon(q_t(x)+h\,\xi_t(x))
    =
    \phi_\varepsilon(q_t(x))
    - h\,\xi_t(x)\,\mathbbm{1}_{\{q_t(x)<\varepsilon\}}
    + o(h).
  \]
  Since the non-differentiability set
  $\{x:\ p(x)\ge \delta,\ q_t(x)=\varepsilon\}$ has Lebesgue measure zero, this identity holds
  almost everywhere on $\{p\ge \delta\}$. Multiplying by
  $p(x)\mathbbm{1}_{\{p(x)\ge\delta\}}$ and integrating gives
  \[
    U_{\delta,\varepsilon}(q_{t+h})
    =
    U_{\delta,\varepsilon}(q_t)
    - h\int_{\Omega_{\delta,\varepsilon}(q_t)} p(x)\,\xi_t(x)\,dx
    + o(h),
  \]
  which is exactly \eqref{eq:soft_undercoverage_decay}. If
  $G_{\mathbf V}(t;\delta,\varepsilon)>0$, the first-order term is strictly negative, so
  $U_{\delta,\varepsilon}(q_{t+h})<U_{\delta,\varepsilon}(q_t)$ for all sufficiently small
  $h>0$.
\end{proof}

\subsection{Proof of Proposition~\ref{prop:regional_repair_fdiv}}

\begin{proof}[Proof of Proposition~\ref{prop:regional_repair_fdiv}]
  Substituting the identity \eqref{eq:density_response_2} into the definition
  \eqref{eq:regional_repair_score} yields
  \[
    G_f(t;\delta,\varepsilon)
    =
    \int_{\Omega_{\delta,\varepsilon}(q_t)}
    p(x)\Bigl[
      -w(r)\,\nabla\cdot(q_t\beta)
      - q_t\,w'(r)\,r\,\|\beta\|^2
    \Bigr]dx.
  \]
  Factoring the integrand gives
  \[
    -w(r)\,\nabla\cdot(q_t\beta) - q_t\,w'(r)\,r\,\|\beta\|^2
    =
    q_t\,w(r)\,\|\beta\|^2
    \left[
      \frac{-\nabla\cdot(q_t\beta)}{q_t\,\|\beta\|^2}
      - \frac{r\,w'(r)}{w(r)}
    \right].
  \]
  This is precisely \eqref{eq:regional_repair_elasticity}.  Points where $\|\beta\|=0$
  contribute zero to the integral and therefore do not affect the identity. Under the stated
  sufficient condition, the integrand in \eqref{eq:regional_repair_elasticity} is nonnegative
  almost everywhere on $\Omega_{\delta,\varepsilon}(q_t)$ and strictly positive on a subset of
  positive measure on which the prefactor $p\,q_t\,w(r)\,\|\beta\|^2$ is positive. Therefore the
  integral is strictly positive, so $G_f(t;\delta,\varepsilon)>0$.
\end{proof}

The entries in Table~\ref{tab:regional_repair_examples} follow by direct substitution of the
corresponding $w(r)$ into \eqref{eq:eta_f} and \eqref{eq:regional_repair_elasticity}.

\subsection{Fixed-Point Analysis: Sufficient Conditions}
\label{app:fixed_points}

\subsubsection{Proof of Proposition~\ref{prop:fixed_point}}

\begin{proof}[Proof of Proposition~\ref{prop:fixed_point}]
  When $w(r) > 0$ for all $r > 0$, the equation
  $\mathbf{V}(x) = w(r) \cdot \beta(x) = \mathbf{0}$ reduces to
  $\beta(x) = \nabla \log r(x) = \mathbf{0}$ on the domain. Hence $\log r(x)$ is constant on
  each connected component, so $p(x)=c\,q(x)$ for some constant $c>0$. Normalization of $p$ and
  $q$ forces $c=1$, and therefore $p=q$.
\end{proof}

\subsubsection{LV Divergence Cases}

For Case~1 ($\nu=q$): degenerate fixed points arise when $\supp(q)$ is partitioned into
$A_1 = \{x: \log\tilde{r}(x) - \E_q[\log\tilde{r}] = 1\}$ and $B_1 = \{x: \beta(x) = \mathbf{0}\}$,
with distinct local proportionality constants $p/q \equiv C_{A_1}$ on $A_1$ and $p/q \equiv C_{B_1}$ on $B_1$.
This illustrates that zeros of the prefactor can coexist with zeros of $\beta$. A sufficient
condition ruling out this particular degeneracy is $\lambda(A_1)=0$ (Lebesgue measure zero)
together with full support.

For Case~2 ($\nu=p$): degenerate fixed points arise on the level set
$A_2 = \{x: \log\tilde{r}(x) - \E_p[\log\tilde{r}] = -1\}$. Mode collapse additionally causes $\alpha_2 \to -\infty$
when $q$ becomes very small, but particles cannot represent mass outside $\supp(q)$. A sufficient
condition excluding the most obvious support degeneracy is full support, for instance via an NF
generator.

For Case~3: the additive structure means $\beta = -\gamma$ is the balance equation. Thus the
equilibrium condition is no longer simply $\beta=0$, and exact coincidence with $p=q$ generally
requires additional alignment between the reference distribution and the current sampler.

% ============================================================
\section{Sinkhorn Normalization: Preventing Repulsion Collapse}
\label{app:sinkhorn}

\paragraph{Collapse of row-softmax at small $\tau$.}
Under the row-softmax coupling, $W_{ij} = \exp(-\|x_i - x_j\|^2/\tau) / \sum_k \exp(-\|x_i - x_k\|^2/\tau)$.
For any fixed particle configuration with distinct pairwise distances, the limit $\tau \to 0$
yields $W_{ij} \to \delta_{j = j^*(i)}$ where $j^*(i) = \arg\min_j \|x_i - x_j\|^2$ is the
nearest neighbor.
The repulsion consequently reduces to $\nabla \log \hat{q}(x_i) \approx (x_{j^*(i)} - x_i)/\tau$,
a nearest-neighbor displacement that carries no information about particles in other modes.
For a multi-modal $q$ with inter-mode distance $\Delta$ and intra-mode spread $\sigma$, particles
in one mode exert zero repulsive force on particles in another whenever $\tau \ll \Delta^2$,
allowing independent mode collapse within each cluster.

\paragraph{Doubly-stochastic coupling maintains global repulsion.}
The Sinkhorn-Knopp algorithm produces the unique doubly-stochastic matrix
$W^* = \mathrm{diag}(u)\, K\, \mathrm{diag}(v)$, $K_{ij} = \exp(-\|x_i - x_j\|^2/\tau)$,
where scaling vectors $u, v > 0$ satisfy $W^* \mathbf{1} = \mathbf{1}/N$ and $W^{*\top} \mathbf{1} = \mathbf{1}/N$.
In practice we implement this normalization in log-space by alternating
\begin{equation}
  L^{(t+1)}_{ij} = L^{(t)}_{ij} - \log\!\sum_{k}\exp(L^{(t)}_{ik}), \qquad
  L^{(t+2)}_{ij} = L^{(t+1)}_{ij} - \log\!\sum_{k}\exp(L^{(t+1)}_{kj}),
  \label{eq:sinkhorn_iter}
\end{equation}
for $T$ steps, followed by a final row normalization to produce the weight matrix $W$.
The column-sum constraint enforces that each source particle $j$ contributes total weight exactly
$1/N$ across all repulsion vectors, preventing any single pair from monopolizing the coupling.
Even at $\tau \ll \Delta^2$, inter-mode pairs retain weight $\Omega(1/N)$, so the mean-shift
vector $W^* \mathbf{x} - x_i$ contains a non-vanishing contribution from distant modes.
The log-space iteration~\eqref{eq:sinkhorn_iter} is numerically stable; $T \ge 10$ iterations are
sufficient for practical convergence in our experiments.

\paragraph{Connection to \citet{he2026sinkhorndriftinggenerativemodels}}
\citet{he2026sinkhorndriftinggenerativemodels} apply doubly-stochastic Sinkhorn normalization to \emph{both}
the attractive coupling $P_{XY}$ (generated-to-real) and the repulsive coupling $P_{XX}$
(generated-to-generated) in a data-based drifting model, showing that this restores the
identifiability property $V = 0 \Leftrightarrow p = q$.
In our data-free setting, the attractive term $\nabla \log p(x) = -\nabla E(x)$ is computed
exactly via autograd and requires no coupling matrix.
We therefore apply Sinkhorn normalization exclusively to the repulsive KDE coupling, preserving
the data-free property while inheriting the collapse-prevention benefit.
The identifiability argument of Appendix~\ref{app:fixed_points} is unaffected, since it
relies on the exactness of $\nabla \log p$ rather than the form of the repulsive weight matrix. $\square$

\end{document}